%% file: ms.tex
\pdfoutput=1

\documentclass[11pt]{article}

\usepackage{acl}

\usepackage{times}
\usepackage{latexsym}

\usepackage[T1]{fontenc}

\usepackage[utf8]{inputenc}

\usepackage{microtype}

\usepackage{amsfonts}       
\usepackage{amsmath}
\usepackage{amsthm}
\usepackage{amssymb}
\usepackage{nicefrac}       
\usepackage{xfrac}
\usepackage{subcaption}
\usepackage{multirow}
\usepackage{ctable} 
\usepackage{booktabs}       
\usepackage{diagbox}
\usepackage{csquotes}
\usepackage{enumitem}
\usepackage{inconsolata}

\newtheorem{theorem}{Theorem}[section]

\newtheorem{defn}{Definition}[section]

\newtheorem{lemma}{Lemma}[section]

\newtheoremstyle{claim}
  {\topsep}
  {\topsep}
  {\itshape}
  {}
  {\itshape}
  {.}
  {.5em}
  {\thmname{#1}\thmnumber{ #2}\thmnote{ (#3)}}
\theoremstyle{claim}

\DeclareMathOperator*{\argmin}{arg\,min}

\title{\texttt{LEATHER}: A Framework for Learning to Generate Human-like Text in Dialogue}

\author{Anthony Sicilia$^1$ \and Malihe Alikhani$^{1,2}$ \\
\{\texttt{anthonysicilia}, \ \texttt{malihe}\}@\texttt{pitt.edu} \\
$^1$Intelligent Systems Program and $^2$Computer Science Department \\
University of Pittsburgh, Pittsburgh, PA, USA
}

\begin{document}
\maketitle
\begin{abstract}
Algorithms for text-generation in dialogue can be misguided. For example, in task-oriented settings, reinforcement learning that optimizes only task-success can lead to abysmal lexical diversity. We hypothesize this is due to poor theoretical understanding of the objectives in text-generation and their relation to the learning process (i.e., model training). To this end, we propose a new theoretical framework for learning to generate text in dialogue. Compared to existing theories of learning, our framework allows for analysis of the multi-faceted goals inherent to text-generation. We use our framework to develop theoretical guarantees for learners that adapt to unseen data. As an example, we apply our theory to study data-shift within a cooperative learning algorithm proposed for the \textit{GuessWhat?!} visual dialogue game. From this insight, we propose a new algorithm, and empirically, we demonstrate our proposal improves both task-success and human-likeness of the generated text. Finally, we show statistics from our theory are empirically predictive of multiple qualities of the generated dialogue, suggesting our theory is useful for model-selection when human evaluations are not available.
\end{abstract}

\section{Introduction}
\label{sec:intro}
\input{01_intro-revised}

\section{Related Works}
\label{sec:related}
\input{02_related}

\section{Theory with Examples}
\label{sec:theory1}
\input{03_theory-revised}


\section{Text-Generation under Data-Shift}
\label{sec:theory-addendum}
\input{03b_addendum-revised}


\section{Experiments}
\label{sec:results}
\input{04_experiments}

\section{Conclusion}
\label{sec:conclusion}
\input{05_conclusion}


\bibliography{anthology,custom}
\bibliographystyle{acl_natbib}
\clearpage
\onecolumn
\appendix
\section{Novel Adaptation Bound and Computation of Energy Statistic}
\label{sec:theory2}
\input{06_theory}
\section{Proofs}
\label{sec:proofs}
\input{07_proofs}
\newpage
\section{Statistics on Dataset}
\label{sec:dataset}
\input{08_details}
\bibliography{anthology,custom}
\end{document}

%% file: 01_intro-revised.tex
\begin{figure*}[t]
    \centering
    \includegraphics[width=\textwidth]{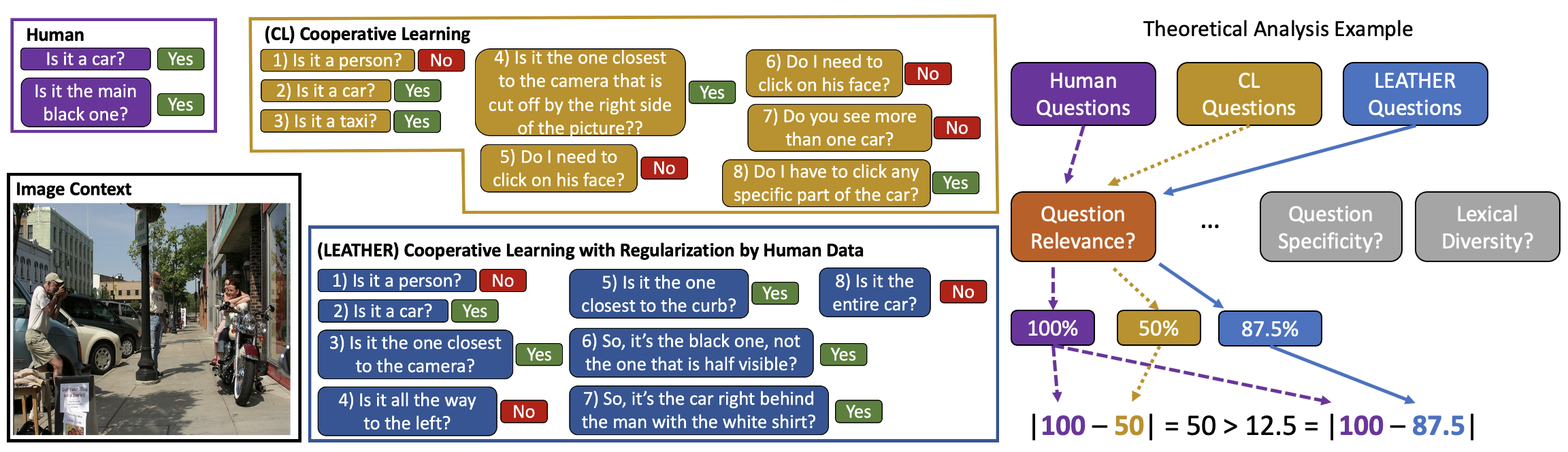}
    \caption{\small Examples of human and generated dialogue with original cooperative learning algorithm \texttt{CL} \citep{shekhar-etal-2019-beyond} and our learning algorithm motivated by our theory (\texttt{LEATHER}). Roughly, \texttt{LEATHER} works by applying a series of tests to generated dialogue and comparing the test results across the human and generated dialogue. Well-generated dialogue is expected to perform similarly to human dialogue on these tests. The example tests the \% of relevant questions. Compared to \texttt{CL}, \texttt{LEATHER} asks more relevant questions and therefore behaves more human-like. Aggregate empirical results in Section~\ref{sec:results} echo this trend.}
    \label{fig:example}
\end{figure*}
Generating coherent, human-like text for dialogue remains a challenge. Yet, it is an inseparable component of open domain and task oriented dialogue systems like Alexa and Siri. Undoubtedly, it is also a complex process to learn. Generation based on classification (e.g., next-word prediction) over-emphasizes the likelihood of text, leading to bland qualities, which are not human-like \citep{holtzman2019curious}. Meanwhile, framing dialogue generation as a Markov decision process is highly data-inefficient when compared to classification \citep{kakade2003sample}. Further, without careful design of rewards, models can suffer from mode-collapse in dialogue, producing repetitive behaviors that are not human-like  \citep{shekhar-etal-2019-beyond}. Even carefully designed rule-based systems are brittle in the presence of unforeseen data-shift.

Theoretical analyses of learning are imperative as they provide solutions to these issues. For example, traditional (PAC) learning theory \citep{valiant1984theory} studies similar issues arising from computational algorithms for learning to classify. Progress in our understanding has been impressive, ranging from comprehensive guarantees on data-efficiency \citep{ shalev2014understanding} to insights for algorithm-design when the learner is faced with data-shift \citep{zhao2019learning, zhang2019bridging, tachet2020domain}. While traditional theory may be applicable to simple generation objectives like next-word prediction, it is  
unfortunately unable to model more diverse goals. That is to say, it is insufficient to study replication of the diverse qualities inherent to human dialogue.

\textit{The goal of this paper is to provide a new theory for analyzing the multi-faceted objectives in computational learning of dialogue generation.} In particular, we propose \texttt{LEATHER}\footnote{\textbf{LEA}rning \textbf{T}heory for \textbf{H}uman-like dialogue gen\textbf{ER}ation} based on existing theories of computational learning. We demonstrate the utility of \texttt{LEATHER} with a focus on understanding data-shift in learning algorithms. We also show empirical results for a task-oriented visual dialogue game. In detail, we contribute as follows:
\begin{enumerate}[leftmargin=*,nolistsep]
\item In Section~\ref{sec:theory1}, we propose \texttt{LEATHER}, our novel theory for computational learning of dialogue generation. We use the \textit{GuessWhat?!} visual dialogue game \citep{de2017guesswhat} as an example to ground abstract terminology in practice. We conclude Section~\ref{sec:theory1} by applying our theory to analyze a cooperative learning algorithm for \textit{GuessWhat?!}. Our theory unveils harmful shifts in data-distribution that occur during training.
\item In Section~\ref{sec:theory-addendum}, we use \texttt{LEATHER} to study the general problem of data-shift in text-generation. We provide new theoretical study that  characterizes \textit{statistical energy} as an effective empirical tool for quantifying the impact of data-shift. Aptly, to conclude Section~\ref{sec:theory-addendum}, we use energy to motivate an improved learning algorithm for our running example -- the \textit{GuessWhat?!} game. 
\item In Section~\ref{sec:results}, empirically, we demonstrate the benefits of our \texttt{LEATHER}-inspired algorithm compared to common baselines. Importantly, we also show our proposed statistic (energy) is predictive of the quality of generated dialogue; i.e., we exhibit a linear relationship. This suggests \texttt{LEATHER} is useful, not only as a theoretical tool for algorithm design, but also as an empirical tool for model-selection.
\end{enumerate}

Our framework is publicly available through experimental code and a Python package.\footnote{ \texttt{\href{https://github.com/anthonysicilia/LEATHER-AACL2022}{github.com/anthonysicilia/LEATHER-AACL2022}}}

%% file: 02_related.tex
\paragraph{Theories of Learning to Generate Text} Most widely, text-generation is framed as a classification problem, in which a model predicts the next word provided existing context (e.g., previous words). While common PAC learning analyses do apply to classification, this theory only describes the learner's ability at the next-word prediction task. In some specific cases, instead, PAC analysis has also been used to analyze high-level objectives and motivate conversational strategies \citep{sicilia2022modeling}, but this analysis is problem-dependent. In contrast, our work offers a general problem-independent formalism for studying high-level qualities of generated text. Another frequent formalism comes from partially observable Markov decision processes (POMDPs) used to motivate reinforcement learning. For example, see \citet{strub2017end}. While POMDPs remedy the issues of typical PAC analysis by supporting implementation of high-level objectives, as we are aware, there are no empirically verified theoretical studies of learning under data-shift in POMDPs. In contrast, we demonstrate \texttt{LEATHER} admits such a theory of learning, using it to predict the human-likeness of generated text under data-shift (where POMDPs fall short).
\paragraph{Theories of Learning with Data-Shift}
Early learning theoretic models of data-shift in classification and regression are due to \citet{ben2010theory, ben2010impossibility} and \citet{mansour2009domain}. While modern approaches are generally similar in spirit, new statistics incorporate increasing information about the learning algorithm \citep{lipton2018detecting, kuroki2019unsupervised, germain2020pac, sicilia2022pac}. Ultimately, such techniques tend to improve the predictive capabilities of a theory in practical application \citep{rabanser2019failing, atwell2022change}.  
Diverse additional approaches to describing the impact of data-shift have also been proposed, for example, using integral probability metrics \citep{redko2017theoretical,redko2020ASO, shen2018wasserstein, johansson2019support}. Unfortunately, existing works focus on classification and regression which, as discussed, are not directly applicable to dialogue generation. Further, this theory does not easily extend to generation (see Section~\ref{sec:statistical_energy}). Ultimately, using \texttt{LEATHER}, our work derives a new statistic (energy) for predicting changes in model performance, which \textit{is} applicable to dialogue generation.

\paragraph{Evaluation of Generated Text} There are many automated metrics for evaluation of generated text including metrics based on $n$-grams such as BLEU \citep{papineni2002bleu}, ROUGE \citep{lin2004rouge}, and CIDEr \citep{vedantam2015cider}. Automated metrics based on neural models are also becoming more prevalent including BLEURT \citep{sellam-etal-2020-bleurt}, BertScore \citep{zhang2019bertscore}, and COSMic~\citep{inan2021cosmic}.
\citet{bruni2017adversarial} propose use of an adversary to discriminate between human and generated text, evaluating based on the generator's ability to fool the adversary. Human annotation and evaluation, of course, remains the gold-standard. Notably, our proposed framework encapsulates these techniques, since it is suitable for analyzing the impact of the learning process on \textit{all of these evaluation strategies and more} (see Section~\ref{sec:theory1} for examples). 

%% file: 03_theory-revised.tex
In this section, we develop our new theoretical framework. To assist our exposition, we use the \textit{GuessWhat?!} visual dialogue game -- a variant of the child's game \textit{I Spy} -- as a running example. We first describe the game along with our modeling interests within the game. We continue with a description of our theory and then apply this theory to analyze an algorithm that learns to play the game.
\subsection{\textit{GuessWhat?!} Visual Dialogue Game}
\label{sec:theory1_guesswhat_example}
An image and \textbf{goal-object} within the image are both randomly chosen.
A \textbf{question-player} with access to the image asks yes/no questions to an \textbf{answer-player} who has access to both the image and goal-object. 
The question-player's goal is to identify the goal-object. The answer-player's goal is to reveal the goal-object to the question-player by answering the yes/no questions appropriately.
The question- and answer-player converse until the question-player is ready to make a guess or at most $m$ questions have been asked.\footnote{By default, $m=8$ following \citet{shekhar-etal-2019-beyond}.} The question-player then guesses which object was the secret goal.
\paragraph{Notation for Human Games}
To discuss this game within our theoretical framework next, we provide some notation. We assume the possible questions, answers, and objects are respectively confined to the sets $\mathcal{Q}$, $\mathcal{A}$, and $\mathcal{O}$. We also assume a set of possible images $\mathcal{I}$. A game between two human players can be represented by a series of random variables. The image-object pair is represented by the random tuple $(I, O)$. The dialogue between the question- and answer-player is represented by the random-tuple $D = (Q_1, A_1, \ldots, Q_P, A_P)$ with some random length $P \leq m$. Each  $Q_i$ is a random question taking value from the set $\mathcal{Q}$ and each $A_i$ is a random answer from the set $\mathcal{A}$. 
\paragraph{Notation for Modeled Games}
From a modeling perspective, in this paper, we focus on the question-player and assume a human answer-player. We consider learning a model that generates the random dialogue $\hat{D} = (\hat{Q}_1, \tilde{A}_1, \ldots \hat{Q}_m, \tilde{A}_m)$ along with a predicted goal object $\hat{O}$.\footnote{Notice, although the answer-player is still human, the answers may follow a distinct distribution due to dependence on the questions, so we demarcate this difference by $\tilde{\square}$.} For example, consider the model of \citet{shekhar-etal-2019-beyond} we study later. It generates dialogue/predicted goal as below:%
\begin{equation}\small
\label{eqn:gen_proc}
\begin{split}
   & \hat{O} = \mathtt{Gues}_\alpha(\mathtt{Enc}_\beta(I, \hat{D})) \\
   & \hat{Q}_{i+1} = \mathtt{QGen}_\theta(\mathtt{Enc}_\beta(I, \hat{Q}_1, \tilde{A}_1, \ldots \hat{Q}_i, \tilde{A}_i)
\end{split}
\end{equation}%
where, aptly, the neural-model $\mathtt{QGen}_\theta : \mathbb{R}^d \to \mathcal{Q}$ is called the \textit{question-generator} and the neural-model $\mathtt{Gues}_\alpha : \mathbb{R}^d \to \mathcal{O}$ is called the \textit{object-guesser}. The final neural-model $\mathtt{Enc}_\beta : \mathcal{I} \times (\mathcal{Q} \times \mathcal{A})^* \to \mathbb{R}^d$ is called the \textit{encoder} and captures pertinent features for the former models to share. Subscripts denote the parameters of each model (to be learned). 
\paragraph{Modeling Goals}
There are two main objectives we consider. The first is task-oriented:%
\begin{equation}\small
\label{eqn:tophase}
    \min\nolimits_{\alpha, \beta} \ \mathbf{E}[1\{\hat{O} \neq O\}]
\end{equation}%
which requires the predicted goal-object align with the true goal. The second objective is more elusive from a mathematical perspective: the generated dialogue $\hat{D}$ should be human-like. That is, it should be similar to the human dialogue $D$. As we see next, our theory is aimed at formalizing this objective.
\subsection{Theoretical Framework (\texttt{LEATHER})}
\label{sec:theory1_framework}
Now, we present our proposed theory with examples from the \textit{GuessWhat?!} game just discussed.
\subsubsection{Terminology}
\label{sec:theory1_terms}
\paragraph{Sets}
Assume a space $\mathcal{C}$, which encompasses the set of dialogue contexts, and a space $\mathcal{D}$, which encompasses the set of possible dialogues. In general, the structure of these sets and representation of elements therein are arbitrary to allow wide applicability to any dialogue system. For particular examples, consider the \textit{Guess What?!} game: $c \in \mathcal{C}$ is an image-goal pair and $d \in \mathcal{D}$ is a list of question-answer pairs. Note, we also allow an additional, arbitrary space $\mathcal{U}$ to account for any unobserved effects on the test outputs (discussed next).
\paragraph{Test Functions}
To evaluate generated text, we assume a group of fixed \textbf{test functions} $\{h_1\ldots h_L\}$ where for each $\ell \in [L]$ the function $h_\ell : \mathcal{D} \times \mathcal{U} \to [0,1]$ assigns a $[0,1]$-valued score that characterizes some high-level property of the dialogue. For example, a test function might be a binary value indicating presence of particular question-type, a continuous value indicating the proportion of clarification questions, a sentiment score, or some other user-evaluation. A test function can also be an automated metric like lexical diversity, for example.
\paragraph{Random Outputs} As noted, the space $\mathcal{U}$ primarily allows the test $h_\ell$ to exhibit randomness due to unobserved effects. For example, this is the case when our test function is a human evaluation and randomness arises from the human annotator. To model this, we assume an unknown distribution $\mathbb{U}$ over $\mathcal{U}$, so that for $U \sim \mathbb{U}$ and dialogue $d \in \mathcal{D}$, the score $h_\ell(d, U)$ is a random variable.  In general, we do not assume too much access to this randomness, since sampling from $\mathbb{U}$ can be costly; e.g., it can require recruiting new annotators or collecting new annotations. Note, $U$ can also be used to encapsulate additional (observable) information needed to conduct the test $h_\ell$ (e.g., a reference dialogue).
\paragraph{Goal Distribution}
Next, we assume a \textbf{goal distribution} $\mathbb{G}$ over the set of contextualized dialogues; i.e., context-dialogue pairs in $\mathcal{C} \times \mathcal{D}$. Typically, $\mathbb{G}$ is the distribution of contextualized dialogues between human interlocutors. In the \textit{GuessWhat?!} example, $\mathbb{G}$ is the distribution of the random, iterated tuple $((I, O), D)$. Recall, $I$ is the random image and $O$ is the random goal-object, which together form the context. $D = (Q_1, A_1\ldots Q_P, A_P)$ is the variable-length tuple of question-answer pairs produced by humans discussing the context $(I, O)$.
\paragraph{Dialogue Learner and Environment} We also assume some \textbf{dialogue learner} parameterized by $\theta \in \mathbb{R}^d$. The learner may only \textit{partially} control each  dialogue -- e.g., the learner might only control a subset of the turns in each dialogue -- and the mechanism through which this occurs is actually unimportant in the general setting; i.e., it will not be assumed in our theoretical results. Ultimately, we need only assume existence of some function $(\theta, c) \overset{\mathsf{E}}{\longrightarrow} \mathbb{P}_\theta(c)$ where $\theta$ are the learned parameters, $c \in \mathcal{C}$ is the context, and $\mathbb{P}_\theta(c)$ is a distribution over dialogues $\mathcal{D}$. In the \textit{GuessWhat?!} example discussed previously, the dialogue learner is $\mathtt{QGen}_\theta$ and the function $\mathsf{E}$ is implicitly defined by Eq.~\eqref{eqn:gen_proc}. In particular, we have $\hat{D} \sim \mathbb{P}_\theta(I, O)$ where image $I$ and object $O$ are sampled from the goal-distribution of contextualized dialogues $((I,O), D) \sim \mathbb{G}$. We call $\mathsf{E}$ the \textbf{environment} of the learner and use \textsf{sans serif} in notation. In the \textit{GuessWhat?!} example, the environment can change for a myriad of reasons: the answer-player could change strategies (inducing a new answer-distribution), the distribution of image $I$ could change, or the distribution of the object $O$ could change. All of which, can impact the function $(\theta, c) \overset{\mathsf{E}}{\longrightarrow} \mathbb{P}_\theta(c)$. One implicit factor we encounter later is the dependence of the environment $\mathsf{E}$ on the encoder parameters $\beta$ in Eq.~\eqref{eqn:gen_proc}. In discussion, we may explicitly write $\mathsf{E}_\beta$ to denote this dependence.%
\paragraph{Formal Objective of Learner}
As discussed before, the conceptual task of the dialogue learner is to produce human-like text. To rephrase more formally: the task of the learner is to induce a contextualized dialogue distribution that is indistinguishable from the the goal distribution. Unfortunately, this objective is made difficult by the complexity of dialogue. In particular, it is unclear what features of the dialogue are important to measure: should we focus on the atomic structure of a dialogue, the overall semantics, or maybe just the fluency? Surely, the answer to this question is dependent on the application. For this reason, we suggest the general notion of a \textit{test function}. Each test $\{h_1\ldots h_L\}$ can be hand selected prior to learning to emphasize a particular goal for the dialogue learner; e.g., as in Figure~\ref{fig:example}, $h_1$ can represent a user evaluation of question relevance, $h_2$ can capture lexical diversity, etc. Then, the quality of the contextualized dialogue distribution induced by the dialogue learner is measured by preservation of the output of the test functions. That is, the output of test functions should be similar when applied to human dialogue about the same context. We capture this idea through the \textbf{test divergence}:%
\begin{equation}\small
\label{eqn:TD}
\begin{split}
    & \mathbf{TD}_\mathsf{E}(\theta) = \sum\nolimits_{\ell=1}^L \mathbf{TD}_\mathsf{E}^\ell(\theta) \\
    & \text{where} \quad \mathbf{TD}_\mathsf{E}^\ell(\theta) = \mathbf{E}[\lvert h_\ell(D, U) - h_\ell(\hat{D}, U) \rvert], \\ 
    & \hspace{3.3em} (C, D) \sim \mathbb{G}, \ \hat{D} \sim \mathbb{P}_\theta(C), \  U \sim \mathbb{U}.
\end{split}
\end{equation}%
Notice, the test divergence is not only dependent on the parameters of the dialogue learner, but also the environment $\mathsf{E}$ which governs the distribution $\mathbb{P}_\theta(C)$. Recall, this function is induced by the learner's environment and its role in eliciting generated dialogue. Finally, with all terms defined, the formal objective of the dialogue learner is typically to minimize the test divergence:%
\begin{equation}\small
\label{eqn:TD_obj}
    \min\nolimits_\theta \ \mathbf{TD}_\mathsf{E}(\theta).
\end{equation}%
\paragraph{Example (BLEU/ROUGE)} Useful examples of test divergence are traditional evaluation metrics, using a human reference -- metrics like BLEU, ROUGE, or accuracy at next-word prediction. To see the connection, in Eq.~\eqref{eqn:TD}, let $L = 1$, let $h_1$ be one of the metrics, and set $U = D$. Then, $h_1(D, U)$ computes some form of $n$-gram overlap between the human reference and itself, so it evaluates to 1 (full overlap). On the other hand, $h_1(\hat{D}, U)$ is the traditional notion of the metric (e.g., BLEU or ROUGE). So, the test divergence simply becomes 1 minus the average of the metric. Notice, this example shows how $U$ can be used to encapsulate observable (random) information as well.
\paragraph{Example (\textit{GuessWhat?!})}
We can also consider a more complicated example in the \textit{GuessWhat?!} game. Here, \citet{shekhar-etal-2019-beyond} evaluate the human-likeness of dialogue with respect to the question strategies. Specifically, the authors consider a group of strategy classifiers $s_i : \mathcal{Q} \to \{0,1\}, i \in [L]$ which each indicate presence of a particular strategy in the input question. For example, $s_1$ might identify if its input is a color question \textit{``Is it blue?''} and $s_2$ might identify if its input is a spatial question \textit{``Is it in the corner?''}. Then, one intuitive mathematical description of the question-strategy dissimilarity may be written%
\begin{equation}\small
\label{eqn:llphase}
\mathbf{E}\Bigg[\sum_{i=1}^\ell \Big \lvert \frac{1}{P}\sum_{j=1}^P s_i(Q_j) - \frac{1}{m}\sum_{k=1}^m s_i(\hat{Q}_k) \Big \rvert \Bigg ]
\end{equation}%
Above captures expected deviation in proportion of color/spatial questions from the human- to the generated-text. 
It also coincides with the definition of test divergence. To see this, note the above is Eq.~\eqref{eqn:TD} precisely when $h_i$ returns the proportion of questions in a dialogue with type identified by $s_i$.%
\paragraph{Example (Human Annotation)} 
Human annotation is also an example, in which, human subjects are presented with two dialogue examples: one machine generated and one from a goal corpus with both dialogues pertaining to the same context. The human then annotates both examples with a score pertaining to the quality of the dialogue (e.g., the relevance of questions as in Figure~\ref{fig:example}). So, $h_i$ is represented by the annotation process, using $U$ to encapsulate any unobserved random effects. Then, the test divergence simply reports average absolute difference between annotations.
\subsection{Application to a \textit{GuessWhat?!} Algorithm}
In this next part, we apply the theory just discussed to analyze a cooperative learning algorithm ($\mathtt{CL}$) proposed by \citet{shekhar-etal-2019-beyond}. Recall Eq.~\eqref{eqn:gen_proc}, $\mathtt{CL}$ generates dialogue/predicted goal as below:%
\begin{equation}\small
\begin{split}
   & \hat{O} = \mathtt{Gues}_\alpha(\mathtt{Enc}_\beta(I, \hat{D})) \\
   & \hat{Q}_{i+1} = \mathtt{QGen}_\theta(\mathtt{Enc}_\beta(I, \hat{Q}_1, \tilde{A}_1, \ldots \hat{Q}_i, \tilde{A}_i)
\end{split}
\end{equation}%
where $\mathtt{QGen}_\theta$ is the question-generator, $\mathtt{Gues}_\alpha$ is the object-guesser, and $\mathtt{Enc}_\beta$ is the encoder.
\paragraph{$\mathtt{CL}$ Algorithm} Conceptually, cooperative learning encompasses a broad class of algorithms in which two or more independent model components coordinate during training to improve each other’s performance. For example, this can involve a shared learning objective \citep{das2017learning}. In the algorithm we consider, \citet{shekhar-etal-2019-beyond} coordinate training of a shared encoder using two distinct learning phases. Written in the context of our theory, they are:
\begin{enumerate}[leftmargin=*,nolistsep]
\item \textbf{Task-Oriented Learning}: Solve Eq.~\eqref{eqn:tophase}. Update $\alpha$ and $\beta$ to minimize $\mathbf{E}[1\{\hat{O} \neq O\}]$.
\item \textbf{Language Learning}: Solve Eq.~\eqref{eqn:TD_obj}. Update $\theta$ and $\beta$ to minimize $\mathbf{TD}_{\mathsf{E}_\beta}(\theta)$ where the test measures accuracy at next-word prediction.
\end{enumerate}
The two phases repeat, alternating until training is finished. As is typical when training neural-networks, the parameter weights are updated using batch SGD with a differentiable surrogate loss. To do so in the \textbf{task-oriented learning phase}, $\mathtt{Gues}_\alpha$ is designed to output probability estimates for each object and the negative log-liklihood of this output distribution is minimized. In the \textbf{language learning phase}, $\mathtt{QGen}_\theta$ is designed to output probabilities for the individual utterances that compose each question. Then, the surrogate optimization is:%
\begin{equation}\small
\label{eqn:suro}
\begin{split}
& \min\nolimits_{\theta, \beta} \mathbf{E}\Big [\sum_{i + 1 \leq P } \ \mathcal{L}(\hat{Q}_{i+1}, Q_{i+1}) \Big ] \quad \text{where} \\
& \hat{Q}_{i+1} = \mathtt{QGen}_\theta(\mathtt{Enc}_\beta(I, Q_1, A_1 \ldots Q_i, A_i)
\end{split}
\end{equation}
and $\mathcal{L}$ sums the negative logliklihood of the individual utterances. Notice, a form of \textit{teacher-forcing} is used in this objective, so that the encoder and question-generator are conditioned on \textit{only} human dialogue during the language learning phase. This fact will become important in the next part.
\paragraph{Problem} 
Importantly, the encoder parameters $\beta$ are updated in \textit{both} the \textit{task-oriented} and \textit{language learning} phases. So, in the language learning phase, the dialogue learner selects $\theta$ to minimize the test divergence in cooperation with a \textit{particular} choice of the encoder parameters -- let us call these $\beta^s$. Then, in the task-oriented learning phase, the learned encoder parameters may change to a new setting $\beta^t$. Importantly, by changing the parameters in Eq.~\eqref{eqn:gen_proc}, we induce a \textit{new} environment $\mathsf{E}_{\beta^t} \neq \mathsf{E}_{\beta^s}$, which governs a new generation process. For brevity, we set $\mathsf{T} = \mathsf{E}_{\beta^t}$ and $\mathsf{S} = \mathsf{E}_{\beta^s}$. This change brings us to our primary issue: the shift in learning environment \textit{does not necessarily preserve the quality of the generated dialogue}. In terms of our formal theory, we rephrase:%
\begin{equation}\small
    \mathbf{TD}_{\mathsf{S}}(\theta) \overset{?}{=} \mathbf{TD}_{\mathsf{T}}(\theta).
\end{equation}%
Without controlling the \textit{change} in test divergence across these two environments, it is possible the two learning phases are not ``cooperating'' at all.
\paragraph{\texttt{LEATHER}-Inspired Solution}
\label{sec:statistical_energy}
In general, it is clear equality will not hold, but we can still ask \textit{how different} these quantities will be. If they are very different, the quality of the dialogue generation learned in the language learning phase may degrade substantially during the task-oriented learning phase. More generally, the problem we see here is a problem of data-shift. In learning theory, the study of data-shift is often referred to as \textit{domain adaptation}. The test divergence on the environment $\mathsf{S}$ -- in which we learn $\theta$ -- is referred to as the \textbf{source error}, while the test divergence on the environment $\mathsf{T}$ -- in which we evaluate $\theta$ -- is referred to as the \textbf{target error}. The tool we use to quantify the change between the source error and the target error is an \textit{adaptation bound}, in which we find a statistic $\Delta$ for which the following is true:\footnote{The inequality is approximate because there are often other statistics in the bound, but through reasonable assumptions, one statistic $\Delta$ is identified as the key quantity of interest. These assumptions should be carefully made to avoid undesirable results \citep{ben2010impossibility, zhao2019learning}.}
\begin{equation}\label{eqn:da}
\small
\mathbf{TD}_{\mathsf{T}}(\theta) \lesssim \mathbf{TD}_{\mathsf{S}}(\theta) + \Delta.
\end{equation}
Then, we can be sure the error in the new environment has not increased much more than $\Delta$. In this sense, we say $\Delta$ is a \textbf{predictive statistic} because it predicts the magnitude of the target error $\mathbf{TD}_{\mathsf{T}}$ from the magnitude of the source error $\mathbf{TD}_{\mathsf{S}}$. To put it more concisely, it predicts the change in error from source to target. \textit{When $\Delta$ is small, the change should be small too or the target error should be even lower than the source error. When $\Delta$ is large, we cannot necessarily come to this conclusion.} Importantly, for $\Delta$ to be useful in practice it should not rely on too much information. In dialogue generation, it is important for $\Delta$ to avoid reliance on the \textit{test functions}, since these can often encompass costly sampling processes like human-evaluation.

As alluded in Section~\ref{sec:related}, many adaptation bounds exist, but as it turns out, none of them are directly applicable to dialogue generation contexts. This is because, as we are aware, computation of all previous bounds relies on efficient access to the test functions $\{h_1\ldots h_L\}$ and samples $U \sim \mathbb{U}$, which is not always possible in dialogue. In particular, these functions, along with the sampling process $U \sim \mathbb{U}$, might represent a time-consuming, real-world processes like human-evaluation. For this reason, in the next section, we prove a new adaptation bound with new statistic $\Delta$, which does not require access to the test functions.

%% file: 03b_addendum-revised.tex
Motivated by the \textit{GuessWhat?!} example and algorithm $\mathtt{CL}$, we continue in this section with a general study of domain adaptation for dialogue generation. We begin by proposing a new (general) adaptation bound for \texttt{LEATHER}. We then apply this general bound to the \textit{GuessWhat?!} algorithm $\mathtt{CL}$, motivating fruitful modifications through our analysis.
\subsection{A Novel Adaptation Bound for \texttt{LEATHER}}
\label{sec:addendum_novel_bound}
\input{03b_addendum}
\subsection{A New Cooperative Learning Algorithm}
\label{sec:theory2_appl}
With all theoretical tools in play, we return to the algorithm $\mathtt{CL}$ and the problem raised in Section~\ref{sec:statistical_energy}.
\paragraph{\texttt{LEATHER}-Motivated Modification}
Recall, we are interested in quantifying and controlling the change in error from source $\mathbf{TD}_{\mathsf{S}}(\theta)$ to target $\mathbf{TD}_{\mathsf{T}}(\theta)$ across the training phases. Based on our theory, we know we should decrease the statistical energy between dialogues to reduce this change. That is, we should reduce the distance between the generated dialogue distributions across learning phases. We hypothesize this may be done by incorporating human dialogue in the task-oriented learning phase. The encoder in $\mathtt{CL}$ sees \textit{no} human dialogue when forming the prediction $\hat{O}$ that is compared to $O$ during task-oriented learning -- as seen in Eq.~\eqref{eqn:gen_proc}, only the generated dialogue $\hat{D}$ is used. In contrast, the encoder sees \textit{only} the human dialogue $D$ in the alternate language learning phase -- i.e., as seen in the surrogate objective in Eq.~\eqref{eqn:suro}. We hypothesize this stark contrast produces large shifts in the parameters $\beta^s \to \beta^t$ between phases. Instead, we propose to \textit{regularize} the task-oriented learning phase with human dialogue as below:%
\begin{equation}\small
\label{eqn:mod}
\begin{split}
    & \min_{\alpha, \beta} \mathbf{E}[1[\hat{O} \neq O]] + \mathbf{E}[1[\hat{O}' \neq O]] \ \quad \text{where} \\
    & \hat{O}' = \mathtt{Gues}_\alpha(\mathtt{Enc}_\beta(I, D)), \quad ((I,O), D) \sim \mathbb{G}
\end{split}
\end{equation}%
and $\hat{O}$ is still as described in Eq.~\eqref{eqn:gen_proc}. Intuitively, this should constrain parameter shift from $\beta^s \to \beta^t$, thereby constraining the change in outputs of the encoder, and ultimately constraining the change in outputs of the question-generator, which is conditioned on the encoder outputs. As the generated dialogue distributions from distinct learning phases will be more similar by this constraint, we hypothesize the penultimate effect will be decreased statistical energy (i.e., since energy measures distance of distributions). Based on our theory, reduced energy provides resolution to our problem: test divergence should be preserved from source to target. 

%% file: 03b_addendum.tex
\paragraph{The Energy Statistic and Computation}
\begin{defn}
\label{defn:energy}
For any independent random variables $A$ and $B$, the discrete energy distance is defined $\varepsilon_{01}(A, B)$ equal to
\begin{equation}\small
\label{eqn:energy}
    2\mathbf{E}[1\{A\neq B\}] - \mathbf{E}[1\{A \neq A'\}] - \mathbf{E}[1\{B \neq B'\}] 
\end{equation}
where $A'$ is an i.i.d copy of $A$, $B'$ is an i.i.d. copy of $B$, and $1\{\cdot\}$ is the indicator function; i.e., it returns 1 for true arguments and 0 otherwise.
\end{defn}
The \textit{discrete energy distance} is a modification of the \textit{energy distance} sometimes called the \textit{statistical energy}. It was first proposed by \citet{szekely1989potential} and was studied extensively by \citet{szekely2013energy} in the case where $A$ and $B$ are continuous variables admitting a probability density function. In general, and especially in dialogue, this is not the case. Aptly, our newly suggested form of the energy distance is more widely applicable to any variables $A$ and $B$ for which equality is defined. While general, this distance can be insensitive, especially when $A$ and $B$ take on many values. To remedy this, we introduce the following.
\begin{defn}
\label{defn:coarse}
Let $\mathcal{D}$ be any set. A coarsening function is a map $c : \mathcal{D} \to \mathcal{D}$ such that $c(\mathcal{D}) = \{c(d) \mid d \in \mathcal{D}\}$ is finite, and further, $\lvert c(\mathcal{D}) \rvert < \lvert \mathcal{D} \rvert$.
\end{defn}
Since $\mathcal{D}$ is likely an immensely large set, this can make the signal $1\{a \neq b\}$ for $a,b \in \mathcal{D}$ overwhelming compared to the signal $1\{a = b\}$, and therefore, weaken the sensitivity of the discrete energy distance, overall. Coarsening functions allow us to alleviate this problem by effectively ``shrinking'' the set $\mathcal{D}$ to a smaller set. To do this, the role of the coarsening function is to exploit additional context to arrive at an appropriate \textit{clustering} of the dialogues, which assigns conceptually ``near'' dialogues to the same cluster. So, the choice of $c(d)$ should be a ``good'' representation of $d$, in the sense that too much valuable information is not lost. As a general shorthand, for a coarsening function $c$ and variables $A, B$, we write
\begin{equation}\small
    \varepsilon_c(A, B) = \varepsilon_{01}(c(A), c(B)).
\end{equation}
In this paper, we implement $c$ using the results of a $k$-means clustering with details in Appendix~\ref{sec:theory2}.
\paragraph{Adaptation Bound}
With these defined, we give the novel bound. Proof of a more general version of this bound -- applicable beyond dialogue contexts (e.g., classification) -- is provided in Appendix~\ref{sec:proofs} Thm.~\ref{thm:main}. 
Notably, our proof requires some technical results on the relationship between discrete energy and the characteristic functions of discrete probability distributions. These may also be of independent interest, outside the scope of this paper.
\begin{theorem}
\label{thm:main_cor}
For any $\theta \in \mathbb{R}^d$, any coarsening function $c : \mathcal{D} \to \mathcal{D}$, and all $\ell \in [L]$%
\begin{equation}\small
    \label{eqn:main-text-bound}
    \mathbf{TD}_\mathsf{T}^\ell(\theta) \leq \gamma + \varphi + \mathbf{TD}_\mathsf{S}^\ell(\theta) + \sqrt{\varepsilon_c(\tilde{D}_1, \tilde{D}_2) \times \delta}
\end{equation}%
where $\tilde{D}_1 \sim \mathbb{P}_\theta(C) = \mathsf{T}(\theta, C), \ \tilde{D}_2 \sim \mathbb{Q}_\theta(C) = \mathsf{S}(\theta, C), \ (C,D) \sim \mathbb{G}, \ U \sim \mathbb{U}$,\footnote{For simplicity, let $\tilde{D}_1, \tilde{D}_2, U$ be pairwise-independent.}%
\begin{equation}\small
    \begin{split}
        & \gamma =  \sum\nolimits_{i \in \{1,2\}} \mathbf{E}[\lvert h_\ell (c(\tilde{D}_i), U) - h_\ell(\tilde{D}_i, U)\rvert] \\ 
        & g \in \argmin_{f \in [0,1]^{\mathcal{D} \times \mathcal{U}}} \sum\nolimits_i \mathbf{E}[\lvert f( c(\tilde{D}_i), U) - h_\ell(D, U)\rvert] \\
        & \quad\text{where}\quad [0,1]^{\mathcal{D} \times \mathcal{U}} = \{f \mid f : \mathcal{D} \times \mathcal{U} \to [0,1]\}. \\
        & \varphi = \sum\nolimits_{i \in \{1,2\}} \mathbf{E}[\lvert g(c(\tilde{D}_i), U) - h_\ell(D, U)\rvert] \\ 
        & \delta = \mathbf{E} \Big [ \sum\nolimits_{x \in c(\mathcal{D})} \lvert g(x, U) - h_\ell(x, U) \rvert \Big ].
    \end{split}
\end{equation}%
\end{theorem}
\paragraph{Unobserved Terms in Dialogue} As noted, an important benefit of our theory is that we need not assume computationally efficient access to the test functions $\{h_1\ldots h_L\}$ or samples $U \sim \mathbb{U}$. Yet, the reader likely notices a number of terms in Eq.~\eqref{eqn:main-text-bound} dependent on both of these. Similar to the traditional case, we argue that our theory is still predictive because it is often appropriate to assume these unobserved terms are small, or otherwise irrelevant. We address each of them in the following: 
\begin{enumerate}[leftmargin=*, nolistsep]
    \item The term $\gamma$ captures average change in test output as a function of the coarsening function $c$. Whenever $c(\tilde{D}_i)$ is a good representative of $\tilde{D}_i$ (i.e., it maintains information to which $h_\ell$ is sensitive) $\gamma$ should be small. 
    \item The next term $\varphi$ is the smallest sum of expected differences that \textit{any} function of the coarsened dialogues $c(\tilde{D}_i)$ and the arbitrary randomness $U$ can achieve in mimicking the true test scores $h_\ell(D, U)$. Since the set of all functions from $\mathcal{D} \times \mathcal{U}$ to $[0,1]$ should be very expressive, this can be seen as another requirement on our coarsened dialogues $c(\tilde{D}_i)$. For example, when $c(\tilde{D}_i) = \tilde{D}_i \approx D$ this term can be close to zero. When instead $\lvert c(\mathcal{D}) \rvert$ is much smaller than $\lvert \mathcal{D} \rvert$ (e.g., a singleton set), we expect $\varphi$ to grow.
    \item The last term $\delta$ can actually be large.
    Fortunately, since $\delta$ is multiplied by the energy distance, this issue is mitigated when the statistical energy is small enough. Ultimately, the energy is paramount in controlling the impact of this term on the bound's overall magnitude.
\end{enumerate}
\paragraph{A Predictive Theory} Granted the background above, our discussion reduces the predictive aspect of the bound to a single key quantity: the discrete energy distance $\varepsilon_c(\tilde{D}_1, \tilde{D}_2)$. In particular, besides the test divergence $\mathbf{TD}_\mathsf{S}$, 
all other terms can be assumed reasonably small by proper choice of the coarsening function, or otherwise controlled by the statistical energy through multiplication. Note, the first issue is discussed in Appendix~\ref{sec:theory2}. Ultimately, the main takeaway is that statistical energy plays the role of $\Delta$ as discussed in Section~\ref{sec:statistical_energy}.

%% file: 04_experiments.tex
\begin{table*}[t]
    \centering\small
    \begin{tabular}{c||c|c|c|c||c|c||c}
         &  $\mathbf{acc}$ $\uparrow$
         &  $\mathbf{lex div}$ $\uparrow$ & $\mathbf{q div}$ $\uparrow$ & $\mathbf{rep q}$ $\downarrow$ & 
         $\mathbf{irr q (TD)}$ $\downarrow$ & $\mathbf{spc q (TD)}$ $\downarrow$ &
         $\mathbf{energy}$ $\downarrow$ \\ \hline 
         \texttt{CL} & 57.1 (55.9) & 9.98 (10.7) & 13.5 (14.3) & 55.9 (58.2) 
         & 30.5 & 23.3 & 0.143 \\
         \texttt{LEATHER} & 58.4 (56.9) & 11.4 (12.7) & 13.1 (16.0) & 53.6 (47.5) & 
         26.2 & 19.5 & 0.123 \\
         \texttt{RL} & 56.3 & 7.3 & 1.04 & 96.5 
         & - & - & -
    \end{tabular}
    \caption{ \small Comparison of \texttt{CL} and our theory-motivated modification \texttt{LEATHER}. Best epoch based on validation \textbf{acc} is reported with last epoch in parentheses. Up/down arrows indicate objective. Metrics are on 100 point scale, excluding \textbf{energy}. The first 4 metrics are automated, the next 2 are from human evaluation, and the last is our proposed statistic. \texttt{LEATHER} improves accuracy and human-likeness of dialogue. Further, our proposed statistic \textbf{energy} is predictive of human-likeness.}
    \label{tab:results}
\end{table*}
\begin{figure}
    \centering
    \includegraphics[width=\columnwidth]{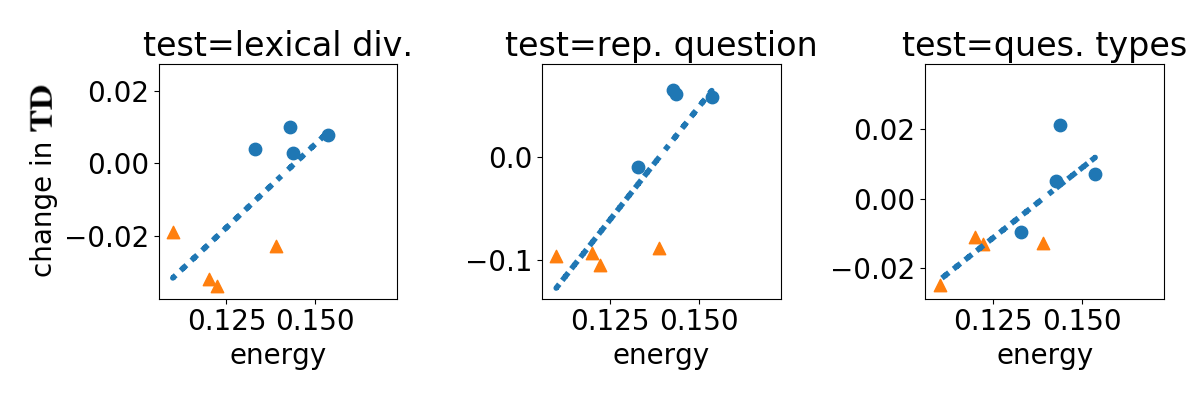}
    \caption{ \small Energy between training phases. Energy is predictive of change in test divergence as desired. Dotted line is line of best fit. Blue circles (\texttt{CL}) indicate use of \textit{only} generated dialogue in task-oriented learning phase. Orange triangles (\texttt{LEATHER}) indicate regularization with human data.}
    \label{fig:energy}
\end{figure}
\subsection{Cooperative Learning via \texttt{LEATHER}}
\label{sec:coop_learn_results}
\paragraph{Setup} 
In general, we use experimental settings of \citet{shekhar-etal-2019-beyond} (e.g., hyperparameters, validation, etc.) with full details available in the code. \texttt{CL} denotes the original algorithm proposed by \citet{shekhar-etal-2019-beyond} (Section~\ref{sec:statistical_energy}). \texttt{LEATHER} denotes our \texttt{LEATHER}-inspired modification (Section~\ref{sec:theory2_appl}).

\paragraph{Automated Metrics} We report average accuracy $\mathbf{acc}$ of the guesser module in identifying the true goal-object across three random seeds as well as average lexical diversity ($\mathbf{lex div}$; type/token ratio over all dialogues), average question diversity ($\mathbf{q div}$; \% unique questions over all dialogues), and average percent of dialogues with verbatim repeated questions ($\mathbf{rep q}$). $\mathbf{acc}$ quantifies task-success, while subsequent metrics are designed to quantify human-likeness of the generated dialogue. These metrics were all previously computed by \citet{shekhar-etal-2019-beyond} with details in their code.
\paragraph{Human Evaluation} We asked two annotators to help us further evaluate the results. Throughout the process, human subject guidelines from the authors' institution were followed and the task was approved by our institution human subject board. 
The annotators examined contextualized human dialogues and generated dialogues from a \texttt{CL} model and \texttt{LEATHER} model. All dialogues used the same image/goal context and annotators observed all dialogues for a specific context in random order without knowing how each dialogue was created. Across 50+ dialogues, average percentage of irrelevant questions per dialogue ($\mathbf{irr q}$) was determined.\footnote{An \textit{irrelevant} question ignores the image or current dialogue context. For example, in Figure~\ref{fig:example}, \texttt{CL} asks about the man's ``face'' (Q5) after learning the goal-object is a car, which ignores dialogue-context. \texttt{CL} also hallucinates an object ``cut off'' on the right side (Q4), which ignores image context.} Average percentage of specific questions ($\mathbf{spc q}$) was also determined.\footnote{A \textit{specific} question contains two or more modifiers of one or more nouns. For example, \texttt{LEATHER} modifies ``car'' with ``behind'' and ``man'' with ``the white shirt'' in Figure~\ref{fig:example} Q7.} We report $\mathbf{TD}$, which gives the average \textit{difference} in percentages from the corresponding human dialogue. Sans scaling, these $\mathbf{TD}$ metrics are examples of the test divergence in Eq.~\eqref{eqn:TD} using a human-evaluation test function. Qualitative analysis of errors was also conducted based on annotator remarks (provided later in this section). 
\paragraph{Impact of \texttt{LEATHER}} In Table~\ref{tab:results}, we compare the cooperative learning algorithms \texttt{CL} and \texttt{LEATHER}. The former uses only the generated dialogue during task-oriented learning, while the latter incorporates human data to regularize the change in parameters underlying the environmental shift. As predicted by our theory, regularization is very beneficial, improving task-success and human-likeness. For example, \texttt{LEATHER} decreases \% of irrelevant questions by 4.8\% compared to \texttt{CL}, which is more similar to human dialogue according to the test divergence ($\mathbf{TD}$). Interestingly, \texttt{LEATHER} also decreased \% of specific questions by 1.7\%. Based on the $\mathbf{TD}$, this is \textit{also} more similar to human dialogue, indicating humans ask fewer specific questions too. The design of the $\mathbf{TD}$ allows us to capture these non-intuitive results. Notably, regularization inspired by \texttt{LEATHER} \textit{allows us to train longer} without degrading task-success or suffering from mode collapse (i.e., repeated questions). Automated human-likeness metrics for the last epoch (in parentheses) show substantial improvements over \texttt{CL} in this case.
\paragraph{Cooperative vs. Reinforcement Learning}
In Table~\ref{tab:results}, we compare the two cooperative learning algorithms \texttt{CL} and \texttt{LEATHER} to the reinforcement learning algorithm (\texttt{RL}). We use the results reported by \citet{shekhar-etal-2019-beyond} for \texttt{RL}, since we share an experimental setup. Compared to \texttt{RL}, both cooperative learning approaches improve task success and human-likeness. As noted in Section~\ref{sec:related}, the theoretical framework for \texttt{RL} (i.e., POMDPs) is not equipped to study interaction of the distinct learning phases within this algorithm (i.e., with respect to data-shift). Better theoretical understanding could explain poor performance and offer improvement as demonstrated with \texttt{LEATHER}, which improves human-likeness of \texttt{CL}.  
\paragraph{Qualitative Analysis}
In dialogue generated by \texttt{CL}, questions with poor relevance ignored the image context (e.g., model hallucination). In dialogue generated by the \texttt{LEATHER} model, irrelevant questions ignored current dialogue context (e.g., a question which should already be inferred from existing answers). We hypothesize this may be due to poor faith in the automated answer-player used for training, which also has problems with model hallucination (e.g., Figure~\ref{fig:example}). Both models had issues with repeated questions. In human dialogue, issues were grammatical with few irrelevant questions. 
\subsection{\texttt{LEATHER} is Empirically Predictive}
Here, we show statistical energy predicts test divergence, empirically. Computation of energy can be automated, so predictive ability is useful for model-selection when human evaluation is not available. We consider test divergence ($\mathbf{TD}$) with 4 groups of tests: (\textbf{A}) the 9 fine-grained strategy classifiers of \citet{shekhar-etal-2019-beyond} used as in Eq.~\eqref{eqn:llphase}, (\textbf{B}) lexical diversity computed as type/token ratio per dialogue, (\textbf{C}) question repetition computed as a binary indicator for each dialogue, and (\textbf{D}) the discussed human-evaluations of question relevance/specificity. Figure~\ref{fig:energy} plots change in $\mathbf{TD}$ for (\textbf{A-C}) as a function of energy. Specifically, change in $\mathbf{TD}$ is the difference $\mathbf{TD}_{\mathsf{T}}(\theta) - \mathbf{TD}_{\mathsf{S}}(\theta)$ where ${\mathsf{S}}$ and ${\mathsf{T}}$ are defined by the transition from language learning to task-oriented learning discussed in Section~\ref{sec:theory1}. We plot this change at the transitions after epochs 65, 75, 85, and 95 (out of 100 total). Notably, \textit{energy is predictive and, specifically, is linearly related to change in test divergence.} For (\textbf{D}), in Table~\ref{tab:results}, we show average energy across all transitions compared to test divergence. Energy is also predictive for these human-evaluation tests.

%% file: 05_conclusion.tex
This work presents \texttt{LEATHER}, a theoretically motivated
framework for learning to generate human-like dialogue.
The energy statistic, which is derived from this theory, is used to analyze \textit{and improve} an algorithm for task-oriented dialogue generation. Further, energy is empirically predictive of improvements in dialogue quality, measured by both automated and human evaluation. 
Future work may involve more experiments to test the utility of \texttt{LEATHER} in other dialogue settings. Theoretically, we hope to study sample-complexity in \texttt{LEATHER}, which is a hallmark of common PAC theories.

\section*{Acknowledgments} We thank the anonymous reviewers for helpful feedback and Jennifer C. Gates, CCC-SLP, for input on qualitative evaluations of dialogue in experiments.

%% file: 06_theory.tex
In this section, we give our novel adaptation bound and details for the accompanying energy statistic. There is some redundancy between this section and Section~\ref{sec:theory-addendum}, but in general, this section is more detailed. Recall, \textit{source} error is denoted $\mathbf{TD}_\mathsf{S}$ and is observed on the environment $\mathbb{Q}_\theta(c) = \mathsf{S}(\theta, c)$. The \textit{target} error is denoted $\mathbf{TD}_\mathsf{T}$ and is observed on the environment $\mathbb{P}_\theta(c) = \mathsf{T}(\theta, c)$. For the algorithm \texttt{CL} discussed in the main text, the target is induced by the task-oriented learning phase and the source is induced by the language learning phase. 
\subsection{The Problem with Traditional Bounds}
\label{sec:problem}
\paragraph{Predictive Adaptation Theories}
An important quality of traditional domain adaptation bounds, proposed for classification and regression problems, is that they offer a \textit{predictive theory}. Namely, without observing the target error $\mathbf{TD}_\mathsf{T}$, we can infer this quantity from $\Delta$ and the source error $\mathbf{TD}_\mathsf{S}$. The utility of this is two-fold: first, it allows us to design algorithms that prepare a learner for data-shift by controlling $\Delta$; second, it allows a practitioner to select an appropriate model to deploy in the presence of data-shift by comparing the different values of $\Delta$ for each model. In general, these use-cases would not be possible without $\Delta$ because the target error $\mathbf{TD}_\mathsf{T}$ \textit{is not observable until it is too late.} In contrast, the quantity $\Delta$ \textit{should} be observable. While this is not always true of $\Delta$, authors typically reduce the main effect of $\Delta$ to one key statistic, which \textit{is} observable. For example, 
\citet{atwell2022change} 
reduce $\Delta$ to one key statistic called the 
$h$-discrepancy
by suggesting the other components making up $\Delta$ are small. This is why we use an ``approximate'' inequality in the main text, since other (small) terms may contribute to the bound.
\paragraph{Traditional Theories Are Not Predictive} Traditional theories of adaptation are \textit{not} predictive for dialogue generation. Namely, computation of $\Delta$ and its key components generally relies on computationally efficient access to the tests $\{h_1 \ldots h_L\}$ and requires sampling from the unknown distribution $U \sim \mathbb{U}$. While we can always \textit{observe} the outputs of $\{h_1\ldots h_L\}$ with randomness $U \sim \mathbb{U}$ through the source error $\mathbf{TD}_\mathsf{S}(\theta)$, it is \textit{not} always the case that we have computational efficiently access to these tests or the randomness. For example, as noted in Section~\ref{sec:theory1_terms}, the group of tests $\{h_1\ldots h_L\}$ along with samples $U$ from the unknown distribution $\mathbb{U}$ may represent complex real-world processes such as human-evaluation. Even for simpler evaluation metrics based on text-classifiers (e.g., like $\{s_1\ldots s_L\}$ in Eq.~\eqref{eqn:llphase}) algorithms for computing $\Delta$ turn out to be non-trivial, and must be handled on a case-by-case basis. Thus, in generation contexts, we typically have no way of computing $\Delta$ algorithmically, and when we do, it can be difficult to implement. If we require an easily implemented, predictive theory, then the classical theory is ruled out. As a solution, we propose a novel adaptation bound. 
\subsection{A Novel Adaptation Bound}
First, we define some terms.
\paragraph{The Energy Statistic and Computation}
\begin{defn}
\label{defn:energy}
For any independent random variables $A$ and $B$, the discrete energy distance is defined:
\begin{equation}\small
\label{eqn:energy}
    \varepsilon_{01}(A, B) = 2\mathbf{E}[1\{A\neq B\}] - \mathbf{E}[1\{A \neq A'\}] - \mathbf{E}[1\{B \neq B'\}] 
\end{equation}
where $A'$ is an i.i.d copy of $A$, $B'$ is an i.i.d. copy of $B$, and $1\{\cdot\}$ is the indicator function; i.e., it returns 1 for true arguments and 0 otherwise.
\end{defn}
The \textit{discrete energy distance} is a modification of the \textit{energy distance} sometimes called the \textit{statistical energy}. It was first proposed by \citet{szekely1989potential} and was studied extensively by \citet{szekely2013energy} in the case where $A$ and $B$ are continuous variables admitting a probability density function. In general, and especially in dialogue, this is not the case. Aptly, we suggest the above form of the energy distance, which is widely applicable to any variables $A$ and $B$ for which equality is defined. While general, this energy distance can be strict and insensitive, especially when $A$ and $B$ take on many possible values. To remedy this, we propose the following addendum.
\begin{defn}
\label{defn:coarse}
Let $\mathcal{D}$ be any set. A coarsening function is a map $c : \mathcal{D} \to \mathcal{D}$ such that $c(\mathcal{D}) = \{c(d) \mid d \in \mathcal{D}\}$ is finite, and further, $\lvert c(\mathcal{D}) \rvert < \lvert \mathcal{D} \rvert$.
\end{defn}
Since $\mathcal{D}$ is likely an immensely large set, this can make the signal $1\{a \neq b\}$ for $a,b \in \mathcal{D}$ overwhelming compared to the signal $1\{a = b\}$, and therefore, weaken the sensitivity of the discrete energy distance, overall. Coarsening functions allow us to alleviate this problem by effectively ``shrinking'' the set $\mathcal{D}$ to a smaller set. To do this, the role of the coarsening function is to exploit additional context to arrive at an appropriate \textit{clustering} of the dialogues, which assigns conceptually ``near'' dialogues to the same cluster. So, the choice of $c(d)$ should be a ``good'' representation of $d$, in the sense that too much valuable information is not lost. As a general shorthand, for a coarsening function $c$ and variables $A, B$, we write
\begin{equation}\small
    \varepsilon_c(A, B) = \varepsilon_{01}(c(A), c(B)).
\end{equation}
\paragraph{Example} One example of a coarsening function for dialogues is $k$-means clustering. In fact, this is the coarsening function we use to compute energy in Section~\ref{sec:results}, selecting $k = 100$. Real-valued vector representations of dialogues (e.g., from model latent space) can capture semantic information about the dialogue \citep{bowman2015generating}, so we use latent space representations (i.e., the output of the encoder) to represent each dialogue and conduct a $k$-means clustering on these representations. For a dialogue $d$ the output $c(d)$ is then defined by the cluster of $d$; i.e., we select an arbitrary dialogue to represent the whole of each cluster and assign this dialogue as the output $c(d)$. In practical implementations, it is typically easier to just compute the energy distance on the cluster labels themselves; this statistic is always equivalent to the energy on the coarsened dialogues, since the map between cluster representatives and cluster labels is bijective. Later, within Lemma~\ref{lem:bijective_for_all}, we prove this equivalence for any bijective map. 

Of course, regardless of implementation, this clustering is dependent on the choice of $k$. Figure~\ref{fig:energy-full} shows that the results in Section~\ref{sec:results} are robust to different choices of $k$. In all cases, there is a linear relationship between the energy and the change in the test divergence.
\begin{figure}
    \centering
    \includegraphics[width=.95\textwidth]{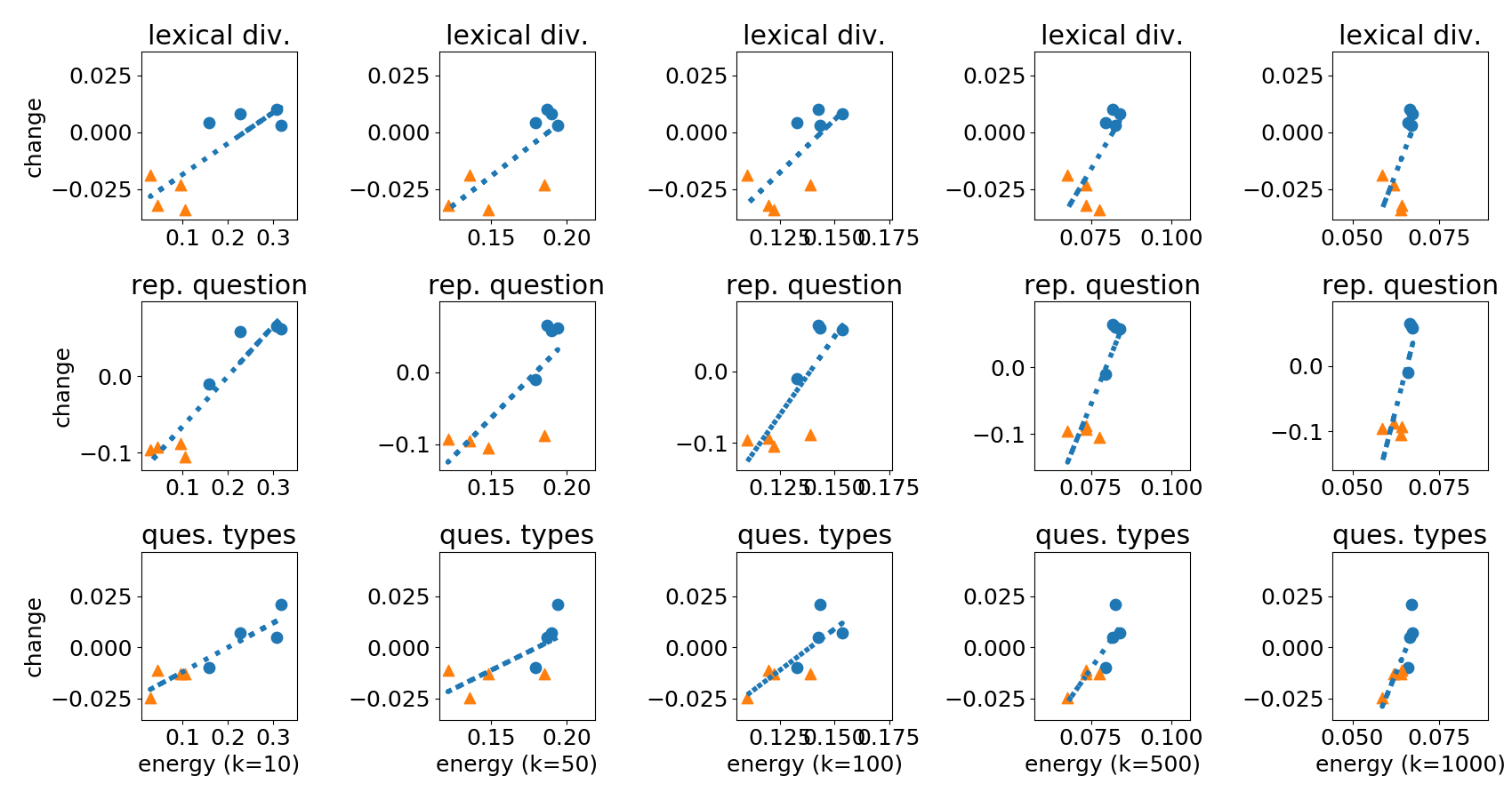}
    \caption{Comparison of energy statistics and automated test functions as in Section~\ref{sec:results}. Here, we vary the parameter $k$ in the $k$-means clustering used to determine the \textit{coarsening function} when computing energy. Trends reported in the main text are robust to variation in $k$.}
    \label{fig:energy-full}
\end{figure}
\paragraph{Adaptation Bound}
With these defined, we give the novel bound.
Proof of a more general version of this bound -- applicable beyond dialogue contexts -- is provided in Appendix~\ref{sec:proofs} Thm.~\ref{thm:main}. In particular, the general version is ``backwards compatible'' in the sense that it also applies to traditional learning theoretic settings like classification and regression. Arguably, in these settings, it also remains more computationally efficient than existing theories. Notably, our proof requires some technical results on the relationship between discrete energy and the characteristic functions of discrete probability distributions. These may also be of independent interest, outside the scope of this paper.
\begin{theorem}
\label{thm:main_cor}
For any $\theta \in \mathbb{R}^d$, any coarsening function $c : \mathcal{D} \to \mathcal{D}$, and all $\ell \in [L]$
\begin{equation}\small
    \label{eqn:bound}
    \mathbf{TD}_\mathsf{T}^\ell(\theta) \leq \gamma + \varphi + \mathbf{TD}_\mathsf{S}^\ell(\theta) + \sqrt{\varepsilon_c(\tilde{D}_1, \tilde{D}_2) \times \delta}
\end{equation}
where $\tilde{D}_1 \sim \mathbb{P}_\theta(C) = \mathsf{T}(\theta, C), \ \tilde{D}_2 \sim \mathbb{Q}_\theta(C) = \mathsf{S}(\theta, C), \ (C,D) \sim \mathbb{G}, \ U \sim \mathbb{U}$,\footnote{For simplicity, let $\tilde{D}_1, \tilde{D}_2, U$ be pairwise-independent. When independence does not hold, similar results can be derived under assumption of context-conditional independence.}
\begin{equation}\small
    \begin{split}
        & \gamma =  \mathbf{E}[\lvert h_\ell (c(\tilde{D}_1), U) - h_\ell(\tilde{D}_1, U)\rvert] + \mathbf{E}[\lvert h_\ell (c(\tilde{D}_2), U) - h_\ell(\tilde{D}_2, U)\rvert] \\ 
        & g \in \argmin_{f \in [0,1]^{\mathcal{D} \times \mathcal{U}}} \sum\nolimits_i \mathbf{E}[\lvert f( c(\tilde{D}_i), U) - h_\ell(D, U)\rvert] \quad\text{where}\quad [0,1]^{\mathcal{X} \times \mathcal{U}} = \{f \mid f : \mathcal{X} \times \mathcal{U} \to [0,1]\}. \\
        & \varphi = \mathbf{E}[\lvert g(c(\tilde{D}_1), U) - h_\ell(D, U)\rvert]  + \mathbf{E}[\lvert g(c(\tilde{D}_2), U) - h_\ell(D, U)\rvert]\\ 
        & \delta = \mathbf{E} \Big [ \sum\nolimits_{x \in c(\mathcal{D})} \lvert g(x, U) - h_\ell(x, U) \rvert \Big ].
    \end{split}
\end{equation}
\end{theorem}
\paragraph{Unobserved Terms in Dialogue} As noted, an important benefit of our theory is that we need not assume computationally efficient access to the test functions $\{h_1\ldots h_L\}$ or samples $U \sim \mathbb{U}$. Yet, the reader likely notices a number of terms in Eq.~\eqref{eqn:bound} dependent on both of these. Similar to the traditional case, we argue that our theory is still predictive because it is typically appropriate to assume these unobserved terms are small, or otherwise irrelevant. We address each of them in the following: 
\begin{enumerate}[leftmargin=*, nolistsep]
    \item The term $\gamma$ captures average change in test output as a function of the coarsening function $c$. Whenever $c(\tilde{D}_i)$ is a good representative of $\tilde{D}_i$ (i.e., it maintains information to which $h_\ell$ is sensitive) $\gamma$ should be small. Since we choose the coarsening function, the former premise is not a strong requirement. In practice, if choice of $c$ is unclear, we recommend studying many choices as in Figure~\ref{fig:energy-full}.
    \item The next term $\varphi$ is the smallest sum of expected differences that \textit{any} function of the coarsened dialogues $c(\tilde{D}_i)$ and the arbitrary randomness $U$ can achieve in mimicking the true test scores $h_\ell(D, U)$. In general, the set of all functions from $\mathcal{D} \times \mathcal{U}$ to $[0,1]$ should be very expressive; e.g., it contains $h_\ell$ itself and any other function which might mimic $h_\ell(D, U)$ better when applied to $c(\tilde{D}_i)$ and $U$. So, it is not unreasonable to expect some good minimizer to exist, and therefore, $\varphi$ to be small. Using this logic, one additional constraint is that $c(\tilde{D}_i)$ has appropriate variance. For instance, if $c(\tilde{D}_i)$ is constant and $D$ is not, $\varphi$ can easily be large. Instead, when $c(\tilde{D}_i)$ does have variance, the expressiveness of the function class $[0,1]^{\mathcal{D} \times \mathcal{U}}$ can be well exploited. For reasonable dialogue learners and a well-chosen $c$, the variance of $c(\tilde{D}_i)$ is a non-issue.
    \item The last term $\delta$ may actually be large, but we argue this is also a non-issue for interpretation purposes. In general, because $\delta$ is an \textit{unnormalized} sum, its magnitude grows with the size of $c(\mathcal{D})$, even if the individual summands may be small. Fortunately, since $\delta$ is multiplied by the energy distance, this issue is mitigated when the statistical energy is small enough. Ultimately, the energy is paramount in controlling the impact of this term on the bound's overall magnitude.
\end{enumerate}
\paragraph{A Predictive Theory} Granted the background above, our discussion reduces the predictive aspect of the bound to a single key quantity: the discrete energy distance $\varepsilon_c(\tilde{D}_1, \tilde{D}_2)$. In particular, besides the test divergence $\mathbf{TD}_\mathsf{S}$ (known prior to the environmental change), all other terms can be assumed reasonably small, or otherwise controlled by the statistical energy through multiplication. Therefore, \textit{if the statistical energy between environments is small, it can be reasonable to assume the dialogue quality has been maintained or improved. Otherwise, it is possible the quality of the generated dialogue has substantially degraded.} In this way, the statistical energy is an easily observable quantity that assists us in determining if the source error $\mathbf{TD}_\mathsf{S}$ known before the environmental change is a good representative of the unknown target error $\mathbf{TD}_\mathsf{T}$, which is observed after the environmental change.
\paragraph{Use Cases}
In general, controlling the statistical energy between dialogues ensures we preserve dialogue quality when the evaluation metrics we care about are not available. As demonstrated in the main text, this makes it useful in algorithm design; i.e., to inform decisions in model training.
Energy can also be useful for model selection. Namely, the generation model whose dialogues have the smallest energy compared to goal dialogue should produce the highest quality dialogue. To see this, simply set $\tilde{D}_2 = D$ in the bound. Similar logical reduction shows the energy is the dominating term in this case as well.

%% file: 07_proofs.tex
In this section we prove the claimed theoretical results. So that the results may be more broadly applicable, we prove them in a more general context and then specify to the context of dialogue generation (in the main text and Appendix~\ref{sec:theory2}).

\subsection{An Adaptation Bound Based on a Discrete Energy Statistic}
In this section, we propose an adaptation bound based on the energy statistic. As we are aware, ours are the first theoretical results relating the statistical energy between distributions to the change in function outputs across said distributions. Given the use of the discrete energy distance (Def.~\ref{defn:energy}) and the accompanying coarsening function (Def.~\ref{defn:coarse}), we appropriately choose to prove our theoretical results for discrete random variables (i.e., those which take on only a countable number of values and exhibit a probability mass function). The effect of this choice is that we also contribute a number of new theoretical results relating the probability mass function of a real-valued, discrete random variable to its characteristic function (i.e., in similar style to the Parseval-Plancherel Theorem). Furthermore, we expand on the relationship between the statistical energy of distributions and their characteristic functions. While this has been well studied in the continuous setting \citep{szekely2013energy} where the distributions of random variables admit probability densities (i.e., absolutely continuous with respect to the Lesbesgue measure), it has not been studied in the case of discrete random variables. We start our results using only \textit{real-valued} discrete variables, but prove our main results for \textit{all} discrete random variables using Lemma~\ref{lem:bijective_for_all}
\subsubsection{Setup}
Suppose $A$ and $B$ are discrete random variables taking on values in $\mathbb{R}^d$ for some $d$. Respectively, the distribution of $A$ is $\alpha$ and the distribution of $B$ is $\beta$. The space $\Omega \subset \mathbb{R}^d$ is the countable subset of $\mathbb{R}^d$ for which $\alpha$ or $\beta$ assigns non-zero probability; i.e., $\Omega = \mathrm{supp}(\alpha) \cup \mathrm{supp}(\beta)$. Then, the expectation of any function $f : \mathbb{R}^d \to \mathbb{R}$ of $A$ is defined:
\begin{equation}
    \mathbf{E}[f(A)] = \int_{\mathbb{R}^d} f\mathrm{d}\alpha = \sum_{a \in \Omega} f(a)p_\alpha(a)
\end{equation}
where $p_\alpha$ is the probability mass function for $A$ (i.e., $\alpha$). Expectations of functions of $B$ are similarly defined.

The \textit{characteristic function} of $A$ is defined as the complex-conjugate of the Fourier-Stieltjes transform of the probability mass function $p_\alpha$. More explicitly, it is the function  $\hat{p}_\alpha : \mathbb{R}^d \to \mathbb{R}$ defined
\begin{equation}
\label{eqn:char_fn_defn}
    \hat{p}_\alpha(\tau) = \mathbf{E}[\mathrm{exp}\{i \tau^\mathrm{T}A\}] = \sum_{a \in \Omega} p_\alpha(a)\mathrm{exp}\{i \tau^\mathrm{T}a\}
\end{equation}
where $i$ is the imaginary unit (i.e., $i^2 = -1$) and $\tau^\mathrm{T}a$ is the (inner) product between column vectors $\tau$ and $a$. Note, the characteristic function always exists and is finite for each $\tau$.
\subsubsection{Parseval-Plancherel Theorem (Reprise)}
One notable use for the \textit{characteristic function} is the following \textit{inversion formula}. In the discrete context we consider, \citet{cuppens1975decomposition} proves the following
\begin{equation}
\label{eqn:cuppens_inversion}
    p_\alpha(a) = \lim_{\tau_1 \to \infty} \lim_{\tau_2 \to \infty} \ldots \lim_{\tau_d \to \infty} \Bigg ( \prod_{i=1}^d 1 / (2\tau_i) \Bigg ) \int_{B(\tau)}  \ \hat{p}_\alpha(t)\mathrm{exp}\{-i t^\mathrm{T}a\} \lambda(\mathrm{d}t)
\end{equation}
where $\tau = (\tau_1, \tau_2, \ldots, \tau_d)^\mathrm{T}$, $B(\tau) = \{x \in \mathbb{R}^d \mid -\tau_i \leq x_i \leq \tau_i\}$, and $\lambda$ is the Lebesgue measure. This inversion formula highlights the connection between the characteristic function and the general Fourier transform as alluded to just before Eq.~\eqref{eqn:char_fn_defn}, since Fourier transforms are well known for their own inversion formulas. Another commonly used result in Fourier Analysis (related to inversion) is the Parseval-Plancherel Theorem. We prove a variation on this result below. As we are aware, it is the first which uses the transform given in Eq.~\eqref{eqn:char_fn_defn} (i.e., specific to discrete, real-valued random variables). 
\begin{lemma}
\label{lem:parseval}
For any discrete random variables $A$ and $B$ as described, taking values in $\mathbb{R}^d$,
\begin{equation}
    \sum_{x \in \Omega} \lvert p_\alpha(x) - p_\beta(x)\rvert^2 = \lim_{\tau_1 \to \infty} \lim_{\tau_2 \to \infty} \ldots \lim_{\tau_d \to \infty} \Bigg ( \prod_{i=1}^d 1 / (2\tau_i) \Bigg ) \int_{B(\tau)} \ \lvert \hat{p}_\alpha(t) - \hat{p}_\beta(t) \rvert^2\lambda (\mathrm{d}t).
\end{equation}
\end{lemma}
\begin{proof}
For any function $f : \mathbb{R}^d \to \mathbb{R}^+$ such that $\sum_{x \in \Omega} f(x) < \infty$ for all $t \in \mathbb{R}^d$, we prove the following more general result
\begin{equation}
\label{eqn:more_general}
    \sum_{x \in \Omega} f^2(x) = \lim_{\tau_1 \to \infty} \lim_{\tau_2 \to \infty} \ldots \lim_{\tau_d \to \infty} \Bigg ( \prod_{i=1}^d 1 / (2\tau_i) \Bigg ) \int_{B(\tau)} \ \hat{f}(x)\hat{f}^*(x)\lambda (\mathrm{d}t)
\end{equation}
where as before a ``hat'' denotes the Fourier-Stieltjes transform given in Eq.~\eqref{eqn:char_fn_defn} and the new notation $\hat{f}^*$ denotes the complex-conjugate of $\hat{f}$. Observe, this proves the desired results because setting $f(x) =  p_\alpha(x) - q_\alpha(x) $ we have 
\begin{equation}
    f^2(x) = (p_\alpha(x) - q_\alpha(x))^2 = \lvert p_\alpha(x) - q_\alpha(x)\rvert^2
\end{equation}
and
\begin{equation}
\begin{split}
    \hat{f}(x)\hat{f}^*(x) & = \widehat{(p_\alpha(x) - p_\alpha(x))}\widehat{(p_\alpha(x) - p_\alpha(x))}^* \\
    & = (\hat{p}_\alpha(x) - \hat{p}_\alpha(x))(\hat{p}_\alpha(x) - \hat{p}_\alpha(x))^* = \lvert \hat{p}_\alpha(x) - \hat{p}_\alpha(x) \rvert^2.
\end{split}
\end{equation}
Proceeding with the proof of Eq.~\eqref{eqn:more_general} we have
\begin{equation}
\begin{split}
    & \lim_{\tau_1 \to \infty} \lim_{\tau_2 \to \infty} \ldots \lim_{\tau_d \to \infty} \Bigg ( \prod_{i=1}^d 1 / (2\tau_i) \Bigg ) \int_{B(\tau)} \ \hat{f}(x)\hat{f}^*(x)\lambda (\mathrm{d}t) \\
    & = \lim_{\tau_i \to \infty} \Bigg ( \prod_{i=1}^d 1 / (2\tau_i) \Bigg ) \int_{B(\tau)} \ \Bigg (\sum_{x \in \Omega} f(x)\mathrm{exp}\{i t^\mathrm{T}x\} \Bigg ) \Bigg ( \sum_{x \in \Omega} f(x)\mathrm{exp}\{-i t^\mathrm{T}x\} \Bigg ) \lambda (\mathrm{d}t) \\
    & = \lim_{\tau_i \to \infty} \Bigg ( \prod_{i} 1 / (2\tau_i) \Bigg ) \int_{B(\tau)} \ \sum_{x \in \Omega} \sum_{x' \in \Omega} f(x)f(x') \mathrm{exp} \{i(t^\mathrm{T}x - t^\mathrm{T}x')\}\lambda (\mathrm{d}t) \qquad \text{(Fubini-Tonelli)} \\
    & = \lim_{\tau_i \to \infty} \Bigg ( \prod_{i=1}^d 1 / (2\tau_i) \Bigg ) \ \sum_{x \in \Omega} \sum_{x' \in \Omega} f(x)f(x') \ \int_{B(\tau)} \mathrm{exp} \{i(t^\mathrm{T}x - t^\mathrm{T}x')\}\lambda (\mathrm{d}t) \qquad \text{(Fubini-Tonelli)} \\
    & = \lim_{\tau_i \to \infty} \sum_{x \in \Omega} \sum_{x' \in \Omega} f(x)f(x') \Bigg ( \prod_{i=1}^d 1 / (2\tau_i) \Bigg ) \Bigg [\ \int_{B(\tau)} \mathrm{exp} \{i(t^\mathrm{T}x - t^\mathrm{T}x')\}\lambda (\mathrm{d}t) \Bigg ] \\
    & = \lim_{\tau_i \to \infty} \sum_{x \in \Omega} \sum_{x' \in \Omega} f(x)f(x') \Bigg ( \prod_{i=1}^d \Bigg [ 1 / (2\tau_i) \ \int_{-\tau_i}^{\tau_i} \mathrm{exp} \{i(t_i(x_i - x'_i)\}\mathrm{d}t_i \Bigg ] \Bigg ) \quad \text{(Fubini-Tonelli)} \\
    & = \lim_{\tau_i \to \infty} \sum_{x \in \Omega} \sum_{x' \in \Omega} f(x)f(x') \Bigg ( \prod_{i=1}^d \chi (x_i, x'_i , \tau_i) \Bigg ) \quad\text{where}\quad \chi = \begin{cases}
    \frac{\sin \tau_i(x_i - x'_i)}{\tau_i (x_i - x'_i)} & \text{if} \ x_i \neq x'_i, \\
    1 & \text{else}
    \end{cases} \\
    & = \sum_{x \in \Omega} \sum_{x' \in \Omega} f(x)f(x') \Bigg ( \lim_{\tau_i \to \infty} \prod_{i=1}^d \chi (x_i, x'_i , \tau_i) \Bigg ) \qquad \text{(DCT)} \\
    & = \sum_{x \in \Omega} \sum_{x' \in \Omega} f(x)f(x') 1[x = x'] \quad\text{where}\quad 1[\mathrm{arg}] = \begin{cases}
    1 & \text{if} \ \mathrm{arg} \ \text{holds}, \\
    0 & \text{else}
    \end{cases} \\
    & = \sum_{x \in \Omega} f^2(x).
\end{split}
\end{equation}
In details: the first equality follows by definition; the second and third by Fubini-Tonelli Theorem;\footnote{The primary assumption of Fubini-Tonelli Theorem requires the \textit{absolute value} of the integrand have finite double or iterated integral/sum. In the first case, with the iterated sum, it is clear for each fixed $t$ since $\sum_x f(x)$ is bounded and so is $\exp\{-iz\}$ for all $z$. In the second and third cases, we simply cite the boundedness of $B(\tau)$ for each fixed $\tau$.} the fourth by simple rules of arithmetic; the fifth again by Fubini-Tonelli Theorem to decompose the volume calculation into a product; the sixth by evaluating the integral; seventh by the dominated convergence theorem;\footnote{The primary assumption of the DCT is that the sequence of functions being integrated (or summed in our case) is dominated by some function $g$ with finite integral (i.e., in the sense that the absolute value of every function in the sequence is less than or equal to $g$ on all inputs). Again, this is easy to see using properties assumed on $f$ and the fact that $|\chi| \leq 1$ for all inputs.} the eighth by evaluating the limit; and the last by simple arithmetic. 
\end{proof}
\subsubsection{The Energy of Discrete Distributions as Described by their Characteristic Functions}
\begin{lemma}
\label{lem:energy_as_char_fn}
For any independent, discrete random variables $A$ and $B$ as described, taking values in $\mathbb{R}^d$,
\begin{equation}
\label{eqn:discrete_energy_integral}
    \varepsilon_{\mathrm{01}}(A, B) = \lim_{\tau_1 \to \infty} \lim_{\tau_2 \to \infty} \ldots \lim_{\tau_d \to \infty} \Bigg ( \prod_{i=1}^d 1 / (2\tau_i) \Bigg ) \int_{B(\tau)} \ \lvert \hat{p}_\alpha(t) - \hat{p}_\beta(t) \rvert^2\lambda (\mathrm{d}t).
\end{equation}
\end{lemma}
\begin{proof}
According to \citet{szekely2013energy}, for independent $A$ and $B$, we have
\begin{equation}
\begin{split}
    &  \lvert \hat{p}_\alpha(t) - \hat{p}_\beta(t) \rvert^2 = \mathbf{E}[\cos \{ t^\mathrm{T}(A - A') \} + \cos \{ t^\mathrm{T}(B - B') \} - \cos \{ t^\mathrm{T}(A - B) \} ] \\
    & = \mathbf{E}\{ 2[1 - \cos \{ t^\mathrm{T}(A - B) \} ] - [1 - \cos \{ t^\mathrm{T}(A - A') \} ] - [1 - \cos \{ t^\mathrm{T}(B - B') \} ]\}
\end{split}
\end{equation}
where $A'$ and $B'$ are i.i.d. copies of $A$ and $B$, respectively. With the equivalence above, by Fubini's Theorem, we may interchange the expectation and integral in Eq.~\eqref{eqn:discrete_energy_integral}. We may also change the order of integration to arrive at
\begin{equation}
\label{eqn:eqn:discrete_energy_integral_szekely}
\begin{split}
    & \lim_{\tau_1 \to \infty} \lim_{\tau_2 \to \infty} \ldots \lim_{\tau_d \to \infty} \Bigg ( \prod_{i=1}^d 1 / (2\tau_i) \Bigg ) \int_{B(\tau)} \ \lvert \hat{p}_\alpha(t) - \hat{p}_\beta(t) \rvert^2\lambda (\mathrm{d}t)  \\
    & = \lim_{\tau_i \to \infty} \mathbf{E} \Bigg [ \Bigg ( \prod_{i=1}^d \frac{1}{(2\tau_i)} \Bigg ) \int_{-\tau_1}^{\tau_1} \ldots \int_{-\tau_d}^{\tau_d} \Big \{ 2 \Big ( 1 - \cos \sum_{i=1}^d \tau_i(A_i - B_i) \Big ) \\
    & \hspace{7em} - \Big ( 1 - \cos \sum_{i=1}^d \tau_i(A_i - A_i') \Big ) - \Big ( 1 - \cos \sum_{i=1}^d \tau_i(B_i - B_i') \Big ) \Big \} \mathrm{d}\tau_d\ldots \mathrm{d}\tau_1 \Bigg].
\end{split}
\end{equation}
To evaluate the integral we first observe, for any $x \in \mathbb{R}^d$,
\begin{equation}
\begin{split}
    \int_{-\tau_d}^{\tau_d}  1 - \cos \sum_{i=1}^d \tau_ix_i \mathrm{d}\tau_d & = 2\tau_d -  \frac{\sin \Big ( \tau_dx_d + \sum_{i=1}^{d-1} \tau_i x_i\Big ) - \sin \Big ( -\tau_dx_d + \sum_{i=1}^{d-1} \tau_i x_i\Big )}{x_d} \\
    & = 2\tau_d - \frac{2\cos\Big (\sum_{i=1}^{d-1} \tau_i x_i\Big )\sin(\tau_dx_d )}{x_d}.
\end{split}
\end{equation}
Notice, the above equation implies an iterative pattern which can be used to solve the multiple integral. Keeping in mind which terms are constants with respect to the differential, we have
\begin{equation}
\label{eqn:multiple_integral}
\begin{split}
    & \int_{-\tau_1}^{\tau_1} \ldots \int_{-\tau_{d-1}}^{\tau_{d-1}} \Big ( \int_{-\tau_d}^{\tau_d}  1 - \cos \sum_{i=1}^d \tau_ix_i \mathrm{d}\tau_d \Big ) \mathrm{d}\tau_{d-1}\ldots\mathrm{d}\tau_1 \\
    & = \int_{-\tau_1}^{\tau_1} \ldots \int_{-\tau_{d-2}}^{\tau_{d-2}} \Bigg ( \int_{-\tau_{d-1}}^{\tau_{d-1}} 2\tau_d - \frac{2\cos\Big (\sum_{i=1}^{d-1} \tau_i x_i\Big )\sin(\tau_dx_d )}{x_d} \mathrm{d}\tau_{d-1} \Bigg )\mathrm{d}\tau_{d-2}\ldots\mathrm{d}\tau_1 \\
    & = \int_{-\tau_1}^{\tau_1} \ldots \int_{-\tau_{d-2}}^{\tau_{d-2}} \Bigg ( (2\tau_d)(2{\tau_{d-1}}) - \frac{4\cos\Big (\sum_{i=1}^{d-2} \tau_i x_i\Big )\sin(\tau_d x_d )\sin(\tau_{d-1}x_{d-1} )}{x_d x_{d-1}} \Bigg )\mathrm{d}\tau_{d-2}\ldots\mathrm{d}\tau_1 \\
    & = \ldots \\
    & = \int_{-\tau_1}^{\tau_1} \ldots \int_{-\tau_{d-j}}^{\tau_{d-j}} \Bigg (\prod_{i=1}^j(2\tau_{d - i + 1}) - \frac{\cos\Big (\sum_{i=1}^{d-j} \tau_i x_i\Big )\prod_{i=1}^j2\sin(\tau_{d - i + 1} x_{d - i + 1} )}{\prod_{i=1}^jx_{d - i + 1}} \Bigg )\mathrm{d}\tau_{d-j}\ldots\mathrm{d}\tau_1 \\
    & \ldots \\
    & = \prod_{i=1}^{d}(2\tau_{d - i + 1}) - \frac{\prod_{i=1}^{d}2\sin(\tau_{d - i + 1} x_{d - i + 1} )}{\prod_{i=1}^{d}x_{d - i + 1} } \\
    & = \prod_{i=1}^{d}(2\tau_{i}) - \frac{\prod_{i=1}^{d}2\sin(\tau_{i} x_{i} )}{\prod_{i=1}^{d}x_{i} } .
\end{split}
\end{equation}
Now, returning to the RHS of Eq.~\eqref{eqn:eqn:discrete_energy_integral_szekely}, linearity of the integral implies
\begin{equation}
\label{eqn:linearity_of_integral}
\begin{split}
& \Bigg ( \prod_{i=1}^d \frac{1}{(2\tau_i)} \Bigg ) \int_{-\tau_1}^{\tau_1} \ldots \int_{-\tau_d}^{\tau_d} \Big \{ 2 \Big ( 1 - \cos \sum_{i=1}^d \tau_i(A_i - B_i) \Big ) \\
& \hspace{7em} - \Big ( 1 - \cos \sum_{i=1}^d \tau_i(A_i - A_i') \Big ) - \Big ( 1 - \cos \sum_{i=1}^d \tau_i(B_i - B_i') \Big ) \Big \} \mathrm{d}\tau_d\ldots \mathrm{d}\tau_1 \\
& = \Bigg ( \prod_{i=1}^d \frac{1}{(2\tau_i)} \Bigg ) \int_{-\tau_1}^{\tau_1} \ldots \int_{-\tau_d}^{\tau_d} \Big \{ 2 \Big ( 1 - \cos \sum_{i=1}^d \tau_i(A_i - B_i) \Big ) \} \mathrm{d}\tau_d\ldots \mathrm{d}\tau_1 \\
& \hspace{3em} - \Bigg ( \prod_{i=1}^d \frac{1}{(2\tau_i)} \Bigg ) \int_{-\tau_1}^{\tau_1} \ldots \int_{-\tau_d}^{\tau_d} \Big \{ \Big ( 1 - \cos \sum_{i=1}^d \tau_i(A_i - A_i') \Big ) \} \mathrm{d}\tau_d\ldots \mathrm{d}\tau_1 \\
& \hspace{3em} - \Bigg ( \prod_{i=1}^d \frac{1}{(2\tau_i)} \Bigg ) \int_{-\tau_1}^{\tau_1} \ldots \int_{-\tau_d}^{\tau_d} \Big \{ \Big ( 1 - \cos \sum_{i=1}^d \tau_i(B_i - B_i') \Big ) \Big \} \mathrm{d}\tau_d\ldots \mathrm{d}\tau_1.
\end{split}
\end{equation}
Thus, we can apply the solution in Eq.~\eqref{eqn:multiple_integral} to solve the integral in Eq.~\eqref{eqn:eqn:discrete_energy_integral_szekely}. Taking $x_i = (A_i - B_i)$ in Eq.~\eqref{eqn:multiple_integral}, we consider the first integral of Eq.~\eqref{eqn:linearity_of_integral} above along with its multiplicative constant: 
\begin{equation}
\begin{split}
    & \Bigg ( \prod_{i=1}^d \frac{1}{(2\tau_i)} \Bigg ) \int_{-\tau_1}^{\tau_1} \ldots \int_{-\tau_d}^{\tau_d} ( 1 - \cos \sum_{i=1}^d \tau_i(A_i - B_i) \Big ) \\
    & = \Bigg ( \prod_{i=1}^d \frac{1}{(2\tau_i)} \Bigg ) \Bigg (\prod_{i=1}^{d}(2\tau_{i}) - \frac{\prod_{i=1}^{d}2\sin \Big \{ \tau_{i}(A_{i} - B_i) \Big \}}{\prod_{i=1}^{d} (A_{i} - B_i) } \Bigg )  \\
    & = 1 - \prod_{i=1}^{d} \frac{\sin \Big \{ \tau_{i}(A_{i} - B_i) \Big \}}{\tau_i (A_{i} - B_i) } = 1 - \prod_{i=1}^d \chi(A_i, B_i, \tau_i)
\end{split}
\end{equation}
where $\chi$ is defined in the proof of Eq.~\eqref{eqn:more_general} (Lemma~\ref{lem:parseval}). Taking $x_i = (A_i - A_i')$ and $x_i = (B_i - B_i')$ and proceeding as above allows us to resolve the entire integral. In particular, we have
\begin{equation}
\begin{split}
    & \lim_{\tau_1 \to \infty} \lim_{\tau_2 \to \infty} \ldots \lim_{\tau_d \to \infty} \Bigg ( \prod_{i=1}^d 1 / (2\tau_i) \Bigg ) \int_{B(\tau)} \ \lvert \hat{p}_\alpha(t) - \hat{p}_\beta(t) \rvert^2\lambda (\mathrm{d}t)  \\
    & = \lim_{\tau_i} \mathbf{E} \Bigg [ 2 \Big (1 - \prod_{i=1}^d \chi(A_i, B_i, \tau_i) \Big ) - \Big (1 - \prod_{i=1}^d \chi(A_i, A_i', \tau_i) \Big ) - \Big (1 - \prod_{i=1}^d \chi(B_i, B_i', \tau_i \Big ) \Bigg] \\
    & =  \mathbf{E} \Bigg [ \lim_{\tau_i} \Bigg \{ 2 \Big (1 - \prod_{i=1}^d \chi(A_i, B_i, \tau_i) \Big ) - \Big (1 - \prod_{i=1}^d \chi(A_i, A_i', \tau_i) \Big ) - \Big (1 - \prod_{i=1}^d \chi(B_i, B_i', \tau_i \Big ) \Bigg \} \Bigg ] \\
    & = \mathbf{E}\big [2 \times 1[A_i \neq B_i] - 1[A_i \neq A'_i] - 1[B_i \neq B'_i] \big ].
\end{split}
\end{equation}
Here, the second equality follows from the dominated convergence theorem and $1[\mathrm{arg}]$ is defined as in proof of Eq.~\eqref{eqn:more_general} (Lemma~\ref{lem:parseval}).
\end{proof}
\subsubsection{Moving from Real-Valued Discrete Variables to Any Discrete Variables}
\begin{lemma}
\label{lem:bijective_for_all}
Let $\tilde{A}$ and $\tilde{B}$ be any independent, discrete random variables over a countable set $\Omega$ (i.e., not necessarily contained in $\mathbb{R}^d$). Then,
\begin{equation}
    \sum_{x \in \Omega} \lvert \tilde{p}_\alpha(x) - \tilde{p}_\beta(x) \rvert = \varepsilon_{01}(\tilde{A}, \tilde{B}).
\end{equation}
where $\tilde{p}_\alpha$ and $\tilde{p}_\beta$ are the mass functions of $\tilde{A}$ and $\tilde{B}$, respectively.
\end{lemma}
\begin{proof}
Let $\Pi \subset \mathbb{R}^d$ with $\lvert \Pi \rvert = \lvert \Omega \rvert$. Note, $\Pi$ exists because $\Omega$ is countable and $\mathbb{R}^d$ is not. Next, let $f : \Omega \to \Pi$ be any bijective map. 

Then, supposing $p_\alpha$ and $p_\beta$ are the mass functions of $f(\tilde{A})$ and $f(\tilde{B})$ respectively, by definition of the pushforward measure, for any $y \in \Pi$ such that $y = f(x)$ for $x \in \Omega$
\begin{equation}
\label{eqn:mass_fns_equal}
    p_\alpha(y) = \tilde{p}_\alpha(\{a \in \Omega \mid f(a) = y\}) = \tilde{p}_\alpha(x).
\end{equation}
Notice, bijectivity of $f$ ensures the last step, because each $y \in \Pi$ has a \textit{unique} inverse $x \in \Omega$. From bijectivity of $f$, we also have injectivity, which implies $1[a\neq b] = 1[f(a) \neq f(b)]$ for all $a,b \in \Omega$. By simple substitution, the previous two facts tells us
\begin{equation}
\begin{split}
    & 2\sum_{a,b \in \Omega} 1[a \neq b] \tilde{p}_\alpha(a)\tilde{p}_\beta(b) - \sum_{a,a' \in \Omega} 1[a \neq a'] \tilde{p}_\alpha(a)\tilde{p}_\alpha(a') - \sum_{b,b' \in \Omega} 1[b \neq b'] \tilde{p}_\beta(b)\tilde{p}_\beta(b') \\
    & = 2\sum_{a,b \in \Omega} 1[f(a) \neq f(b)] p_\alpha(f(a))p_\beta(f(b)) - \sum_{a,a' \in \Omega} 1[f(a) \neq f(a)'] p_\alpha(f(a))p_\alpha(f(a')) \\
    &\hspace{3em} - \sum_{b,b' \in \Omega} 1[f(b) \neq f(b')] p_\beta(f(b))p_\beta(f(b'))
\end{split}
\end{equation}
Since $f$ is surjective too (i.e., along with injective), summation of any function $g(f(a), f(b))$ over $a,b \in \Omega$ and summation of $g(c,d)$ over $c,d \in \Pi$ are equivalent.\footnote{In particular, because $f$ is surjective, we know all pairs $(c,d) \in \Pi^2$ have some pair $(a,b) \in \Omega^2$ for which $(f(a),f(b)) = (c,d)$; i.e., we do not ``miss'' a term in this sum. Because $f$ is injective, we know all pairs $(c,d) \in \Pi^2$ have \textit{only one} pair $(a,b) \in \Omega^2$ for which $(f(a),f(b)) = (c,d)$; i.e., we do not ``repeat'' a term in this sum.} So, we can continue as follows:
\begin{equation}
\begin{split}
    & 2\sum_{a,b \in \Omega} 1[f(a) \neq f(b)] p_\alpha(f(a))p_\beta(f(b)) - \sum_{a,a' \in \Omega} 1[f(a) \neq f(a)'] p_\alpha(f(a))p_\alpha(f(a')) \\
    &\hspace{3em} - \sum_{b,b' \in \Omega} 1[f(b) \neq f(b')] p_\beta(f(b))p_\beta(f(b')) \\
    & = 2\sum_{c,d \in \Pi} 1[c \neq d] p_\alpha(c)p_\beta(d) - \sum_{c,c' \in \Omega} 1[c \neq c'] p_\alpha(c)p_\alpha(c') - \sum_{d,d' \in \Omega} 1[d \neq d'] p_\beta(d)p_\beta(d') \\
\end{split}
\end{equation}
In other words, the previous two equations tell us $\varepsilon_{01}(\tilde{A}, \tilde{B}) = \varepsilon_{01}(f(\tilde{A}), f(\tilde{B}))$. Applying equivalence of the mass functions, then Lemmas~\ref{lem:parseval} and \ref{lem:energy_as_char_fn}, then equivalence of the energies:
\begin{equation}
    \sum_{x \in \Omega} \lvert \tilde{p}_\alpha(x) - \tilde{p}_\beta(x) \rvert = \sum_{y \in \Pi } \lvert p_\alpha(y) - p_\beta(y) \rvert = \varepsilon_{01}(f(\tilde{A}), f(\tilde{B})) = \varepsilon_{01}(\tilde{A}, \tilde{B}).  
\end{equation}
Note, this uses the fact that functions of independent random variables are also independent.
\end{proof}
\subsubsection{The Main Bound}
\begin{theorem}
\label{thm:main}
Let $A$ and $B$ be any independent random variables over any space $\mathcal{X}$ and let $S$, $S'$ be random variables over $[0,1]$. Let $U$ be a random variable, independent from $A$ and $B$, over any set $\mathcal{U}$. Suppose $c : \mathcal{X} \to \Omega$ is a coarsening function (so, $\Omega \subset \mathcal{X})$ and let $f \in [0,1]^{\mathcal{X} \times \mathcal{U}}$. Then,
\begin{equation}
     \mathbf{E}[\lvert S - f(A, U)\rvert] \leq \gamma + \varphi + \mathbf{E}[\lvert S' - f(B, U)\rvert] + \sqrt{\varepsilon_c(A, B) \times \delta}
\end{equation}
where
\begin{equation}
\begin{split}
     \gamma & = \mathbf{E}[\lvert f(c(B), U) - f(B)\rvert] + \mathbf{E}[\lvert f(c(A), U) - f(A)\rvert], \\
     g & \in \argmin_{h \in [0,1]^{\mathcal{X} \times \mathcal{U}} }\mathbf{E}[\lvert S - h(c(A), U) \rvert] + \mathbf{E}[\lvert h(c(B), U) - S' \rvert], \\
     \varphi & = \mathbf{E}[\lvert S - g(c(A), U) \rvert] + \mathbf{E}[\lvert g(c(B), U) - S' \rvert], \\
     \delta & = \sum_{x \in \Omega} \lvert g(x) - f(x) \rvert^2
\end{split}
\end{equation}
\end{theorem}
\begin{proof}
For any $g \in {[0,1]}^{\mathcal{X} \times \mathcal{U}}$, by way of the triangle inequality and monotonicity of the expectation,
\begin{equation}
\label{eqn:large_triangle_ineq}
\begin{split}
    & \mathbf{E}[\lvert S - f(A, U)\rvert] = \mathbf{E}[\lvert S - f(A, U)\rvert] + \mathbf{E}[\lvert S' - f(B, U)\rvert] - \mathbf{E}[\lvert S' - f(B, U)\rvert] \\
    & = \mathbf{E}[\lvert S - g(c(A), U) + g(c(A), U) - f(A, U)\rvert] + \mathbf{E}[\lvert S' - f(B, U)\rvert] - \mathbf{E}[\lvert S' - f(B, U)\rvert] \\
    & \leq \mathbf{E}[\lvert S - g(c(A), U) \rvert] + \mathbf{E}[\lvert g(c(A), U) - f(A, U)\rvert] + \mathbf{E}[\lvert S' - f(B, U)\rvert] \\
    & \hspace{3em}- \mathbf{E}[\lvert S' - f(B, U)\rvert] \\
    & \leq \mathbf{E}[\lvert S - g(c(A), U) \rvert] + \mathbf{E}[\lvert g(c(A), U) - f(A, U)\rvert] + \mathbf{E}[\lvert S' - f(B, U)\rvert] \\ 
    & \hspace{3em} - \mathbf{E}[\lvert g(c(B), U) - f(B, U)\rvert] + \mathbf{E}[\lvert g(c(B), U) - S' \rvert] \\
    & \leq \mathbf{E}[\lvert S - g(c(A), U) \rvert] + \mathbf{E}[\lvert g(c(A), U) - f(c(A), U) \rvert] + \mathbf{E}[\lvert f(c(A), U) - f(A, U)\rvert] \\ 
    & \hspace{3em} + \mathbf{E}[\lvert S' - f(B, U)\rvert] - \mathbf{E}[\lvert g(c(B), U) - f(B, U)\rvert] + \mathbf{E}[\lvert g(c(B), U) - S' \rvert] \\
    & \leq \mathbf{E}[\lvert S - g(c(A), U) \rvert] + \mathbf{E}[\lvert g(c(A), U) - f(c(A), U)\rvert] + \mathbf{E}[\lvert f(c(A), U) - f(A, U)\rvert] \\ 
    & \hspace{3em}+ \mathbf{E}[\lvert S' - f(B, U)\rvert] - \mathbf{E}[\lvert g(c(B, U) - f(c(B), U) \rvert] \\
    & \hspace{3em} + \mathbf{E}[\lvert f(c(B), U) - f(B, U)\rvert] + \mathbf{E}[\lvert g(c(B), U) - S' \rvert]. \\
\end{split}
\end{equation}
Set $\tilde{B} = c(B)$, $\tilde{A} = c(A)$ and set
\begin{equation}
\begin{split}
    \gamma & = \mathbf{E}[\lvert f(\tilde{B}, U) - f(B, U)\rvert] + \mathbf{E}[\lvert f(\tilde{A}, U) - f(A, U)\rvert], \\
    g & \in \argmin_{h \in [0,1]^{\mathcal{X}\times\mathcal{U}} }\mathbf{E}[\lvert S - h(\tilde{A}, U) \rvert ] + \mathbf{E}[\lvert h(\tilde{B}, U) - S' \rvert], \\
    \varphi & = \mathbf{E}[\lvert S - g(\tilde{A}, U) \rvert ] + \mathbf{E}[\lvert g(\tilde{B}, U) - S' \rvert].
\end{split}
\end{equation}
Then, Eq.~\eqref{eqn:large_triangle_ineq} implies
\begin{equation}
    \mathbf{E}[\lvert S - f(A, U)\rvert] \leq \gamma + \varphi + \mathbf{E}[\lvert S' - f(B, U)\rvert] + \mathbf{E}[\lvert g(\tilde{A}, U) - f(\tilde{A}, U)\rvert] - \mathbf{E}[\lvert g(\tilde{B}, U) - f(\tilde{B}, U)\rvert].
\end{equation}
Now, suppose $\tilde{p}_\alpha$ and $\tilde{p}_\beta$ are probability mass functions for $\tilde{A}$ and $\tilde{B}$, respectively. Then, using basic properties of the expectation along with other noted facts,
\begin{equation}
\begin{split}
    & \mathbf{E}[\lvert g(\tilde{A}, U) - f(\tilde{A}, U)\rvert] - \mathbf{E}[\lvert g(\tilde{B}, U) - f(\tilde{B}, U)\rvert] \\
    & = \mathbf{E}\Big [ \sum_{a \in \Omega} \lvert g(a, U) - f(a, U)\rvert \tilde{p}_\alpha(a) - \sum_{b \in \Omega} \lvert g(b, U) - f(b, U)\rvert \tilde{p}_\beta(b) \Big ] \quad\text{(Fubini)}\\
    & = \mathbf{E}\Big [ \sum_{x \in \Omega} \lvert g(x, U) - f(x, U)\rvert (\tilde{p}_\alpha(x) - \tilde{p}_\beta(x)) \Big ] \leq \mathbf{E}\Big [ \sum_{x \in \Omega} \lvert g(x, U) - f(x, U)\rvert \lvert \tilde{p}_\alpha(x) - \tilde{p}_\beta(x) \rvert \Big ]\\
    & \leq \mathbf{E} \Bigg [\Bigg ( \sum_{x \in \Omega} \lvert g(x, U) - f(x, U) \rvert^2 \Bigg )^{1/2} \Bigg ( \sum_{x\in\Omega} \lvert \tilde{p}_\alpha(x) - \tilde{p}_\beta(x) \rvert^2 \Bigg )^{1/2} \ \Bigg ]\qquad \text{(Cauchy-Schwarz)} \\
    & \leq \sqrt{\varepsilon_{01}(\tilde{A}, \tilde{B})}  \times \mathbf{E} \Bigg [ \Bigg ( \sum_{x \in \Omega} \lvert g(x, U) - f(x, U) \rvert^2 \Bigg )^{1/2} \ \Bigg ] \qquad \text{(Lemma~\ref{lem:bijective_for_all})}
\end{split}
\end{equation}
In the last step, we may apply Lemma~\ref{lem:bijective_for_all} because $\tilde{A}$ and $\tilde{B}$ are still independent (i.e., they are functions of independent random variables) and are now discrete too. Defining $\delta$ appropriately yields the result.
\end{proof}
\subsubsection{Proof of Thm.~\ref{thm:main_cor} and Other Applications of Thm.~\ref{thm:main}}
\paragraph{Thm.~\ref{thm:main_cor}}
Thm.~\ref{thm:main_cor} is simply a specification of Thm.~\ref{thm:main} above. In fact, it is better stated as a corollary of Thm.~\ref{thm:main}. We set $\mathcal{X} = \mathcal{D}$, leave $\mathcal{U}$ and its variable $U$ unchanged, and set $S = S' = h_\ell(D, U)$. Then, $A = \tilde{D}_1$ and $B = \tilde{D}_2$. Taking $f = h_\ell$ yields the result. 
\paragraph{Classification and Regression}
In adaptation for classification and regression, we consider a source distribution $\mathbb{S}$ governing random variables $(X_S, Y_S)$ and a target distribution $\mathbb{T}$ governing random variables  $(X_T, Y_T)$. In general, the goal is to predict $Y_\square$ from $X_\square$. We can set $S = Y_T$ and $S' = Y_S$. We may also set $A = X_T$ and $B = X_S$. Then, we learn $f$ from a pre-specified \textit{hypothesis class} $\mathcal{H} \subseteq [0,1]^{\mathcal{X} \times \mathcal{U}}$. Typically, $U$ is ignored in these settings, but it seems possible to employ this term to model stochastic (Gibbs) predictors; i.e., in PAC-Bayesian Frameworks \citep{germain2020pac, sicilia2022pac}. 
Notice, for regression, our framework only considers a normalized response variable and the mean absolute error. 
\subsubsection{Sample Complexity}
As alluded in Section~\ref{sec:conclusion}, a key shortcoming of our framework compared to existing frameworks is the absence of any terms measuring \textit{sample-complexity}. That is, we do not explicitly quantify the difference between our empirical observation of the energy and the \textit{true} energy (i.e., the \textit{population} version of the statistic) using the number of samples in our observation. This is a big part of computational learning theory, as the act of choosing a function $f$ \textit{using data} -- or, in dialogue contexts, choosing the parameter $\theta$ using data -- can have significant impact on the difference between our observations of a statistical processes and reality. In fact, this impact is the basis of overfitting and, besides computational efficiency, is the main pillar of study in traditional PAC learning\footnote{Probably Approximately Correct learning} \citep{valiant1984theory, shalev2014understanding}. In more recent studies of domain adaptation, like our work, the population-only bound can be just as important for purpose of understanding and interpretation. Furthermore, if we only care about the empirical samples in-hand, these population-only bounds are directly applicable,\footnote{The empirical sample becomes the whole population about which we are concerned.} which partly explains the empirical effectiveness of our theory in Section~\ref{sec:results}. Nonetheless, the role of sample-complexity can be very informative and useful in practice \citep{perez2021tighter}
and would be important for model-selection applications as described at the end of Appendix~\ref{sec:theory2}. We leave investigation of sample-complexity as future work. As we are aware, there is currently no appropriate description of sample-complexity for dialogue generation contexts.

%% file: 08_details.tex
\begin{table}[h]
    \centering\footnotesize
    \begin{tabular}{cccccc}
         \textbf{unique images} & \textbf{unique objects} &  \textbf{words (+1 occurrences)} & \textbf{words (+3 occurrences)} & \textbf{questions} \\ \toprule
         67K & 134K & 19K & 6.6K & 277K
    \end{tabular}
    \caption{\small Statistics on \textit{GuessWhat?!}. For more information (e.g., train/test splits) see original proposal \citep{de2017guesswhat}. }
    \label{tab:datastats}
\end{table}
\begin{figure*}[h]
    \centering
    \includegraphics[width=\textwidth]{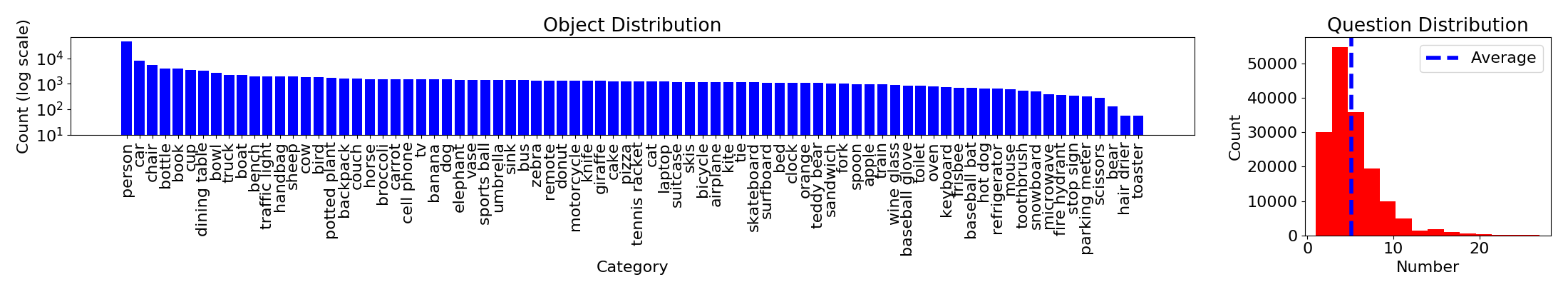}
    \caption{\small Visualization of object counts and dialogue length in \textit{GuessWhat?!} dataset.}
    \label{fig:orig-datastats}
\end{figure*}

%% file: ms.bbl
\begin{thebibliography}{37}
\expandafter\ifx\csname natexlab\endcsname\relax\def\natexlab#1{#1}\fi

\bibitem[{Atwell et~al.(2022)Atwell, Sicilia, Hwang, and
  Alikhani}]{atwell2022change}
Katherine Atwell, Anthony Sicilia, Seong~Jae Hwang, and Malihe Alikhani. 2022.
\newblock The change that matters in discourse parsing: Estimating the impact
  of domain shift on parser error.
\newblock In \emph{Findings of the Association for Computational Linguistics:
  ACL 2022}, pages 824--845.

\bibitem[{Ben-David et~al.(2010{\natexlab{a}})Ben-David, Blitzer, Crammer,
  Kulesza, Pereira, and Vaughan}]{ben2010theory}
Shai Ben-David, John Blitzer, Koby Crammer, Alex Kulesza, Fernando Pereira, and
  Jennifer~Wortman Vaughan. 2010{\natexlab{a}}.
\newblock A theory of learning from different domains.
\newblock \emph{Machine learning}, 79(1):151--175.

\bibitem[{Ben-David et~al.(2010{\natexlab{b}})Ben-David, Lu, Luu, and
  P{\'a}l}]{ben2010impossibility}
Shai Ben-David, Tyler Lu, Teresa Luu, and D{\'a}vid P{\'a}l.
  2010{\natexlab{b}}.
\newblock Impossibility theorems for domain adaptation.
\newblock In \emph{Proceedings of the Thirteenth International Conference on
  Artificial Intelligence and Statistics}, pages 129--136. JMLR Workshop and
  Conference Proceedings.

\bibitem[{Bowman et~al.(2015)Bowman, Vilnis, Vinyals, Dai, Jozefowicz, and
  Bengio}]{bowman2015generating}
Samuel~R Bowman, Luke Vilnis, Oriol Vinyals, Andrew~M Dai, Rafal Jozefowicz,
  and Samy Bengio. 2015.
\newblock Generating sentences from a continuous space.
\newblock \emph{arXiv preprint arXiv:1511.06349}.

\bibitem[{Bruni and Fernandez(2017)}]{bruni2017adversarial}
Elia Bruni and Raquel Fernandez. 2017.
\newblock Adversarial evaluation for open-domain dialogue generation.
\newblock In \emph{Proceedings of the 18th Annual SIGdial Meeting on Discourse
  and Dialogue}, pages 284--288.

\bibitem[{Cuppens(1975)}]{cuppens1975decomposition}
R~Cuppens. 1975.
\newblock Decomposition of multivariate distributions.

\bibitem[{Das et~al.(2017)Das, Kottur, Moura, Lee, and Batra}]{das2017learning}
Abhishek Das, Satwik Kottur, Jos{\'e}~MF Moura, Stefan Lee, and Dhruv Batra.
  2017.
\newblock Learning cooperative visual dialog agents with deep reinforcement
  learning.
\newblock In \emph{Proceedings of the IEEE international conference on computer
  vision}, pages 2951--2960.

\bibitem[{De~Vries et~al.(2017)De~Vries, Strub, Chandar, Pietquin, Larochelle,
  and Courville}]{de2017guesswhat}
Harm De~Vries, Florian Strub, Sarath Chandar, Olivier Pietquin, Hugo
  Larochelle, and Aaron Courville. 2017.
\newblock Guesswhat?! visual object discovery through multi-modal dialogue.
\newblock In \emph{Proceedings of the IEEE Conference on Computer Vision and
  Pattern Recognition}, pages 5503--5512.

\bibitem[{Germain et~al.(2020)Germain, Habrard, Laviolette, and
  Morvant}]{germain2020pac}
Pascal Germain, Amaury Habrard, Fran{\c{c}}ois Laviolette, and Emilie Morvant.
  2020.
\newblock Pac-bayes and domain adaptation.
\newblock \emph{Neurocomputing}, 379:379--397.

\bibitem[{Holtzman et~al.(2019)Holtzman, Buys, Du, Forbes, and
  Choi}]{holtzman2019curious}
Ari Holtzman, Jan Buys, Li~Du, Maxwell Forbes, and Yejin Choi. 2019.
\newblock The curious case of neural text degeneration.
\newblock In \emph{International Conference on Learning Representations}.

\bibitem[{Inan et~al.(2021)Inan, Sharma, Khalid, Soricut, Stone, and
  Alikhani}]{inan2021cosmic}
Mert Inan, Piyush Sharma, Baber Khalid, Radu Soricut, Matthew Stone, and Malihe
  Alikhani. 2021.
\newblock Cosmic: A coherence-aware generation metric for image descriptions.
\newblock \emph{arXiv preprint arXiv:2109.05281}.

\bibitem[{Johansson et~al.(2019)Johansson, Sontag, and
  Ranganath}]{johansson2019support}
Fredrik~D Johansson, David Sontag, and Rajesh Ranganath. 2019.
\newblock Support and invertibility in domain-invariant representations.
\newblock In \emph{The 22nd International Conference on Artificial Intelligence
  and Statistics}, pages 527--536. PMLR.

\bibitem[{Kakade(2003)}]{kakade2003sample}
Sham~Machandranath Kakade. 2003.
\newblock \emph{On the sample complexity of reinforcement learning}.
\newblock University of London, University College London (United Kingdom).

\bibitem[{Kuroki et~al.(2019)Kuroki, Charoenphakdee, Bao, Honda, Sato, and
  Sugiyama}]{kuroki2019unsupervised}
Seiichi Kuroki, Nontawat Charoenphakdee, Han Bao, Junya Honda, Issei Sato, and
  Masashi Sugiyama. 2019.
\newblock Unsupervised domain adaptation based on source-guided discrepancy.
\newblock In \emph{Proceedings of the AAAI Conference on Artificial
  Intelligence}, volume~33, pages 4122--4129.

\bibitem[{Lin(2004)}]{lin2004rouge}
Chin-Yew Lin. 2004.
\newblock Rouge: A package for automatic evaluation of summaries.
\newblock In \emph{Text summarization branches out}, pages 74--81.

\bibitem[{Lipton et~al.(2018)Lipton, Wang, and Smola}]{lipton2018detecting}
Zachary Lipton, Yu-Xiang Wang, and Alexander Smola. 2018.
\newblock Detecting and correcting for label shift with black box predictors.
\newblock In \emph{International conference on machine learning}, pages
  3122--3130. PMLR.

\bibitem[{Mansour et~al.(2009)Mansour, Mohri, and
  Rostamizadeh}]{mansour2009domain}
Yishay Mansour, Mehryar Mohri, and Afshin Rostamizadeh. 2009.
\newblock Domain adaptation: Learning bounds and algorithms.
\newblock \emph{arXiv preprint arXiv:0902.3430}.

\bibitem[{Papineni et~al.(2002)Papineni, Roukos, Ward, and
  Zhu}]{papineni2002bleu}
Kishore Papineni, Salim Roukos, Todd Ward, and Wei-Jing Zhu. 2002.
\newblock Bleu: a method for automatic evaluation of machine translation.
\newblock In \emph{Proceedings of the 40th annual meeting of the Association
  for Computational Linguistics}, pages 311--318.

\bibitem[{P{\'e}rez-Ortiz et~al.(2021)P{\'e}rez-Ortiz, Rivasplata,
  Shawe-Taylor, and Szepesv{\'a}ri}]{perez2021tighter}
Mar{\'\i}a P{\'e}rez-Ortiz, Omar Rivasplata, John Shawe-Taylor, and Csaba
  Szepesv{\'a}ri. 2021.
\newblock Tighter risk certificates for neural networks.
\newblock \emph{Journal of Machine Learning Research}, 22.

\bibitem[{Rabanser et~al.(2019)Rabanser, G{\"u}nnemann, and
  Lipton}]{rabanser2019failing}
Stephan Rabanser, Stephan G{\"u}nnemann, and Zachary Lipton. 2019.
\newblock Failing loudly: An empirical study of methods for detecting dataset
  shift.
\newblock \emph{Advances in Neural Information Processing Systems}, 32.

\bibitem[{Redko et~al.(2017)Redko, Habrard, and Sebban}]{redko2017theoretical}
Ievgen Redko, Amaury Habrard, and Marc Sebban. 2017.
\newblock Theoretical analysis of domain adaptation with optimal transport.
\newblock In \emph{ECML PKDD}, pages 737--753. Springer.

\bibitem[{Redko et~al.(2020)Redko, Morvant, Habrard, Sebban, and
  Bennani}]{redko2020ASO}
Ievgen Redko, Emilie Morvant, Amaury Habrard, Marc Sebban, and Youn{\`e}s
  Bennani. 2020.
\newblock A survey on domain adaptation theory.
\newblock \emph{ArXiv}, abs/2004.11829.

\bibitem[{Sellam et~al.(2020)Sellam, Das, and Parikh}]{sellam-etal-2020-bleurt}
Thibault Sellam, Dipanjan Das, and Ankur Parikh. 2020.
\newblock \href {https://doi.org/10.18653/v1/2020.acl-main.704} {{BLEURT}:
  Learning robust metrics for text generation}.
\newblock In \emph{Proceedings of the 58th Annual Meeting of the Association
  for Computational Linguistics}, pages 7881--7892, Online. Association for
  Computational Linguistics.

\bibitem[{Shalev-Shwartz and Ben-David(2014)}]{shalev2014understanding}
Shai Shalev-Shwartz and Shai Ben-David. 2014.
\newblock \emph{Understanding machine learning: From theory to algorithms}.
\newblock Cambridge university press.

\bibitem[{Shekhar et~al.(2019)Shekhar, Venkatesh, Baumg{\"a}rtner, Bruni,
  Plank, Bernardi, and Fern{\'a}ndez}]{shekhar-etal-2019-beyond}
Ravi Shekhar, Aashish Venkatesh, Tim Baumg{\"a}rtner, Elia Bruni, Barbara
  Plank, Raffaella Bernardi, and Raquel Fern{\'a}ndez. 2019.
\newblock \href {https://doi.org/10.18653/v1/N19-1265} {Beyond task success: A
  closer look at jointly learning to see, ask, and {G}uess{W}hat}.
\newblock In \emph{Proceedings of the 2019 Conference of the North {A}merican
  Chapter of the Association for Computational Linguistics: Human Language
  Technologies, Volume 1 (Long and Short Papers)}, pages 2578--2587,
  Minneapolis, Minnesota. Association for Computational Linguistics.

\bibitem[{Shen et~al.(2018)Shen, Qu, Zhang, and Yu}]{shen2018wasserstein}
Jian Shen, Yanru Qu, Weinan Zhang, and Yong Yu. 2018.
\newblock Wasserstein distance guided representation learning for domain
  adaptation.
\newblock In \emph{AAAI}.

\bibitem[{Sicilia et~al.(2022{\natexlab{a}})Sicilia, Atwell, Alikhani, and
  Hwang}]{sicilia2022pac}
Anthony Sicilia, Katherine Atwell, Malihe Alikhani, and Seong~Jae Hwang.
  2022{\natexlab{a}}.
\newblock Pac-bayesian domain adaptation bounds for multiclass learners.
\newblock In \emph{The 38th Conference on Uncertainty in Artificial
  Intelligence}.

\bibitem[{Sicilia et~al.(2022{\natexlab{b}})Sicilia, Maidment, Healy, and
  Alikhani}]{sicilia2022modeling}
Anthony Sicilia, Tristan Maidment, Pat Healy, and Malihe Alikhani.
  2022{\natexlab{b}}.
\newblock Modeling non-cooperative dialogue: Theoretical and empirical
  insights.
\newblock \emph{Transactions of the Association for Computational Linguistics},
  10:1084--1102.

\bibitem[{Strub et~al.(2017)Strub, De~Vries, Mary, Piot, Courvile, and
  Pietquin}]{strub2017end}
Florian Strub, Harm De~Vries, Jeremie Mary, Bilal Piot, Aaron Courvile, and
  Olivier Pietquin. 2017.
\newblock End-to-end optimization of goal-driven and visually grounded dialogue
  systems.
\newblock In \emph{Proceedings of the 26th International Joint Conference on
  Artificial Intelligence}, pages 2765--2771.

\bibitem[{Szekely(1989)}]{szekely1989potential}
Gabor~J Szekely. 1989.
\newblock Potential and kinetic energy in statistics.
\newblock \emph{Lecture Notes, Budapest Institute}.

\bibitem[{Sz{\'e}kely and Rizzo(2013)}]{szekely2013energy}
G{\'a}bor~J Sz{\'e}kely and Maria~L Rizzo. 2013.
\newblock Energy statistics: A class of statistics based on distances.
\newblock \emph{Journal of statistical planning and inference},
  143(8):1249--1272.

\bibitem[{Tachet~des Combes et~al.(2020)Tachet~des Combes, Zhao, Wang, and
  Gordon}]{tachet2020domain}
Remi Tachet~des Combes, Han Zhao, Yu-Xiang Wang, and Geoffrey~J Gordon. 2020.
\newblock Domain adaptation with conditional distribution matching and
  generalized label shift.
\newblock \emph{Advances in Neural Information Processing Systems},
  33:19276--19289.

\bibitem[{Valiant(1984)}]{valiant1984theory}
Leslie~G Valiant. 1984.
\newblock A theory of the learnable.
\newblock \emph{Communications of the ACM}, 27(11):1134--1142.

\bibitem[{Vedantam et~al.(2015)Vedantam, Lawrence~Zitnick, and
  Parikh}]{vedantam2015cider}
Ramakrishna Vedantam, C~Lawrence~Zitnick, and Devi Parikh. 2015.
\newblock Cider: Consensus-based image description evaluation.
\newblock In \emph{Proceedings of the IEEE conference on computer vision and
  pattern recognition}, pages 4566--4575.

\bibitem[{Zhang et~al.(2019{\natexlab{a}})Zhang, Kishore, Wu, Weinberger, and
  Artzi}]{zhang2019bertscore}
Tianyi Zhang, Varsha Kishore, Felix Wu, Kilian~Q Weinberger, and Yoav Artzi.
  2019{\natexlab{a}}.
\newblock Bertscore: Evaluating text generation with bert.
\newblock \emph{arXiv preprint arXiv:1904.09675}.

\bibitem[{Zhang et~al.(2019{\natexlab{b}})Zhang, Liu, Long, and
  Jordan}]{zhang2019bridging}
Yuchen Zhang, Tianle Liu, Mingsheng Long, and Michael Jordan.
  2019{\natexlab{b}}.
\newblock Bridging theory and algorithm for domain adaptation.
\newblock In \emph{International Conference on Machine Learning}, pages
  7404--7413. PMLR.

\bibitem[{Zhao et~al.(2019)Zhao, Des~Combes, Zhang, and
  Gordon}]{zhao2019learning}
Han Zhao, Remi~Tachet Des~Combes, Kun Zhang, and Geoffrey Gordon. 2019.
\newblock On learning invariant representations for domain adaptation.
\newblock In \emph{International Conference on Machine Learning}, pages
  7523--7532. PMLR.

\end{thebibliography}
